%% file: main.tex
\documentclass{article}

\usepackage{natbib}

\usepackage[eandd, nonanonymous]{style}

\usepackage[utf8]{inputenc} 
\usepackage[T1]{fontenc}    
\usepackage{hyperref}       
\usepackage{url}            
\usepackage{booktabs}       
\usepackage{amsfonts}       
\usepackage{nicefrac}       
\usepackage{microtype}      
\usepackage{xcolor}         
\usepackage{stmaryrd}

\usepackage{amsmath} 
\usepackage{ amssymb }
\usepackage{amsfonts}
\usepackage[most]{tcolorbox}

\usepackage{pifont} 

\usepackage[]{algorithm2e}
\usepackage{ulem} 
\usepackage{caption} 

\usepackage{cleveref}
\usepackage{mathrsfs}
\usepackage[OT1]{fontenc}

\title{Rank and computation of the pathlifting Jacobian of a DAG ReLU network 
}

 \author{} %

\usepackage{todonotes}

\newcommand\RGcmtout[1]{}
\newcommand\ALcmtout[1]{}
\input{notations_rank}

\input{fancy_theorem}
\input{propandtheorem.tex}

\begin{document}

\maketitle
\begin{abstract}
This paper provides a self-contained proof of the rank of the pathlifting Jacobian of a DAG ReLU network by induction on the network's number of hidden nodes. 
In fact, the induction is elementary, and the key recipe is to consider the skeleton matrix of the network, a sparse matrix encoding the network paths, and transform the representation of one of its hidden neurons into an output node. 
The proof relies on intermediate propositions which link the pathlifting, its Jacobian, the network parameters, and its skeleton matrix, which, on top of allowing us to determine the rank of the pathlifting Jacobian, also provide a way to compute it without backpropagation and whose computation cost is super efficient in practice compare to usual backpropagation. 
The paper is provided with a \href{https://gitlab.inria.fr/mverbock/pathliftingandskeleton/-/tree/main?ref_type=heads}{Python module} that implements the different propositions of the paper for feed-forward networks and is used to experimentally quantify the computational gain of computing the pathlifting Jacobian with the proposed theory.
\end{abstract}
\input{Introduction}

\input{Theorems}

\input{SketchOfProof.tex}
\input{Experiment.tex}

\input{Conclusion.tex}

\bibliographystyle{plainnat}
\bibliography{bibli}
\newpage
\appendix
\input{Proofs.tex}

\end{document}

%% file: notations_rank.tex
\usepackage{xcolor}
\definecolor{colorEum}{HTML}{bfbf40}
\definecolor{colorEup}{HTML}{d22da9}
\definecolor{colorEkm}{HTML}{33cccc}
\definecolor{colore}{HTML}{2644d9}

\newcommand{\flower}{\mathrel{\text{\ding{95}}}}

\newcommand{\R}{\mathbb{R}}

\newcommand{\Skeleton}{B}

\newcommand{\Paths}{\mathcal{P}}

\newcommand{\dimx}{r}
\newcommand{\dimy}{q}
\newcommand{\dimTheta}{d}

\newcommand{\dimPathL}{P}

\newcommand{\dimNode}{m}

\newcommand{\PathL}{\Phi}
\newcommand{\NN}{\mathcal{F}}
\newcommand{\Nodes}{N}
\newcommand{\Edges}{E}
\newcommand{\edge}{e}
\newcommand{\node}{n}
\newcommand{\Graph}{\mathcal{G}}
\newcommand{\pathn}{p}

\newcommand{\deep}{L}
\newcommand{\width}{w}

\newcommand{\rk}{{\normalfont \text{rk}}}

\newcommand{\diag}{{\normalfont \text{diag}}}

\newcommand{\scalarp}[1]{\left<#1\right>}
\newcommand{\trans}[1]{\ensuremath{#1^{\top}}}

\newcommand{\norm}[1]{\ensuremath{\left|\left|#1\right|\right|}}

\DeclareMathOperator*{\argmin}{arg\,min}
\DeclareMathOperator*{\sign}{sign}

\newcommand{\VecOnes}{
    \colorlet{oldcolor}{.}
   \begin{matrix}
        1 \\
        \vdots \\
        1 
    \end{matrix}}

\newcommand{\VecmOnes}{
    \colorlet{oldcolor}{.}
   \begin{matrix}
        -1 \\
        \vdots \\
        -1 
    \end{matrix}}

\newcommand{\VecZeros}{
    \colorlet{oldcolor}{.}
   \begin{matrix}
        0 \\
        \vdots \\
        0 
    \end{matrix}}

\newcommand{\VecDot}{
    \colorlet{oldcolor}{.}
   \begin{matrix}
        \hdots \\
        \ddots \\
        \hdots
    \end{matrix}}

\newcommand{\OrderSkeleton}{\begin{pmatrix}
            &{\color{black}\VecOnes} & {\color{blue}\VecZeros} & {\color{blue}\VecDot} & {\color{blue}\VecZeros} & {\color{black}A_1} & {\color{magenta}\mathbf{0}}&M_1&\\
             &{\color{blue}\VecZeros} & {\color{black}\VecOnes} & {\color{blue}\VecDot} & {\color{blue}\VecZeros} & {\color{black}A_2} & {\color{magenta}\mathbf{0}}&M_2&\\
             &\cdot & \cdot & \ddots & \cdot & \vdots & {\color{magenta}\vdots} &\cdot&\\
            &{\color{blue}\VecZeros} & {\color{blue}\VecDot} & {\color{blue}\VecZeros} & {\color{black}\VecOnes} & {\color{black}A_{-1}} & {\color{magenta}\mathbf{0}}&M_{-1}&\\
            &{\color{green}\mathbf{0}} & {\color{green}\mathbf{0}} &{\color{green}\mathbf{0}} &{\color{green}\mathbf{0}} & {\color{green}\mathbf{0}} & D&\cdot&
        \end{pmatrix}}

\newcommand{\IntermediatSkeleton}{\begin{pmatrix}
            &{\color{red}\VecZeros} & {\color{black}\VecZeros} & {\color{black}\VecDot} & {\color{black}\VecZeros} & {\color{black}A_1} & {\color{magenta}\mathbf{0}}&M_1&\\
             &{\color{red}\VecmOnes} & {\color{black}\VecOnes} & {\color{black}\VecDot} & {\color{black}\VecZeros} & {\color{black}A_2} & {\color{magenta}\mathbf{0}}&M_2&\\
             &\cdot & \cdot & \cdot & \cdot & \cdot & \color{magenta}{\vdots} &\cdot&\\
            &{\color{red}\VecmOnes} & {\color{black}\VecDot} & {\color{black}\VecZeros} & {\color{black}\VecOnes} & {\color{black}A_{-1}} & {\color{magenta}\mathbf{0}}&M_{-1}&\\
            &{\color{green}\mathbf{0}} & {\color{green}\mathbf{0}} &{\color{green}\mathbf{0}} &{\color{green}\mathbf{0}} & {\color{green}\mathbf{0}} & D&\cdot&
        \end{pmatrix}}

\newcommand{\ModifiedSkeleton}{\begin{pmatrix}
            &{\color{red}\VecZeros} & {\color{black}\VecZeros} & {\color{black}\VecDot} & {\color{black}\VecZeros} & {\color{black}A_1} & {\color{magenta}\mathbf{0}}&M_1&\\
             &{\color{red}\VecZeros} & {\color{black}\VecOnes} & {\color{black}\VecDot} & {\color{black}\VecZeros} & {\color{black}A_2} & {\color{magenta}\mathbf{0}}&M_2&\\
             &\cdot & \cdot & \cdot & \cdot & \cdot & \color{magenta}{\vdots} &\cdot&\\
            &{\color{red}\VecZeros} & {\color{black}\VecDot} & {\color{black}\VecZeros} & {\color{black}\VecOnes} & {\color{black}A_{-1}} & {\color{magenta}\mathbf{0}}&M_{-1}&\\
            &{\color{green}\mathbf{0}} & {\color{green}\mathbf{0}} &{\color{green}\mathbf{0}} &{\color{green}\mathbf{0}} & {\color{green}\mathbf{0}} & D& \cdot&
        \end{pmatrix}}

\newcommand{\DiagZOnes}{ \begin{pmatrix}
        1 & 0 & \cdot & \cdot\\
        \vdots & \vdots & \cdot & \cdot\\
        1 & 0 & \cdot & \cdot\\
        0 & 1 & 0& \cdot\\
        \vdots & \vdots & \vdots & \cdot\\
        0 & 1 & 0& \cdot\\
        \cdot & \cdot & \ddots & \cdot\\
        \cdot & \cdot & 0 & 1\\
        \cdot & \cdot & \vdots & \vdots\\
        \cdot & \cdot & 0 & 1\\
    \end{pmatrix}} 

%% file: fancy_theorem.tex
\usepackage{amsthm}

\newtheorem{theorem}{Theorem}
\newtheorem{definition}{Definition}
\newtheorem{corollary}{Corollary}
\newtheorem{proposition}{Proposition}
\newtheorem{lemma}{Lemma}
\newtheorem{remark}{Remark}

\newtheoremstyle{upright}   
  {10pt}                    
  {10pt}                    
  {}                        
  {}                        
  {\bfseries}               
  {.}                       
  { }                       
  {}                        

\theoremstyle{upright}

\tcbuselibrary{theorems}
\definecolor{theoremblue}{HTML}{2196F3} 
\definecolor{theoremred}{HTML}{f03f09} 
\definecolor{Defblue}{HTML}{0000A0} 
\definecolor{Backgroundblue}{HTML}{0000A0} 
\definecolor{proofpurple}{HTML}{9013FE} 
\definecolor{defyellow}{HTML}{FFEB3B} 
\definecolor{assumptionred}{HTML}{ff3b3b} 
\definecolor{conventionred}{HTML}{ff3b3b} 
\definecolor{remarkblack}{HTML}{000000} 
\definecolor{remarkorange}{HTML}{dc5800} 
\definecolor{examplegreen}{HTML}{00695c} 
\definecolor{schemarouge}{HTML}{EE220C} 
\definecolor{schemableu}{HTML}{4599EF} 

\tcolorboxenvironment{theorem}{
  enhanced jigsaw,
  sharp corners,
  colframe=theoremblue,
  colback=theoremblue!2,
  coltitle=black,
  fonttitle=\bfseries,
  borderline west={2pt}{0pt}{theoremred},
  before skip=10pt,
  after skip=10pt,
  boxrule=0pt,
  left=10pt,
  right=10pt,
  breakable
}

\tcolorboxenvironment{definition}{
  enhanced jigsaw,
  sharp corners,
  colframe=Defblue,
  colback=Backgroundblue!2,
  coltitle=black,
  fonttitle=\bfseries,
  borderline west={2pt}{0pt}{Defblue},
  before skip=10pt,
  after skip=10pt,
  boxrule=0pt,
  left=10pt,
  right=10pt,
  breakable
}

\tcolorboxenvironment{corollary}{
  enhanced jigsaw,
  sharp corners,
  colframe=theoremblue,
  colback=Backgroundblue!2,
  coltitle=black,
  fonttitle=\bfseries,
  borderline west={2pt}{0pt}{theoremblue},
  before skip=10pt,
  after skip=10pt,
  boxrule=0pt,
  left=10pt,
  right=10pt,
  breakable
}

\tcolorboxenvironment{proposition}{
  enhanced jigsaw,
  sharp corners,
  colframe=theoremblue,
  colback=Backgroundblue!2,
  coltitle=black,
  fonttitle=\bfseries,
  borderline west={2pt}{0pt}{theoremblue},
  before skip=10pt,
  after skip=10pt,
  boxrule=0pt,
  left=10pt,
  right=10pt,
  breakable
}

\tcolorboxenvironment{condition}{
  enhanced jigsaw,
  sharp corners,
  colframe=theoremblue,
  colback=theoremblue!2,
  coltitle=black,
  fonttitle=\bfseries,
  borderline west={2pt}{0pt}{theoremblue},
  before skip=10pt,
  after skip=10pt,
  boxrule=0pt,
  left=10pt,
  right=10pt,
  breakable
}

\tcolorboxenvironment{example}{
  enhanced jigsaw,
  sharp corners,
  colframe=examplegreen,
  colback=examplegreen!2,
  coltitle=black,
  fonttitle=\bfseries,
  borderline west={2pt}{0pt}{examplegreen},
  before skip=10pt,
  after skip=10pt,
  boxrule=0pt,
  left=10pt,
  right=10pt,
  breakable
}

\tcolorboxenvironment{lemma}{
  enhanced jigsaw,
  sharp corners,
  colframe=theoremblue,
  colback=Backgroundblue!2,
  coltitle=black,
  fonttitle=\bfseries,
  borderline west={2pt}{0pt}{theoremblue},
  before skip=10pt,
  after skip=10pt,
  boxrule=0pt,
  left=10pt,
  right=10pt,
  breakable
}

\tcolorboxenvironment{assumption}{
  enhanced jigsaw,
  sharp corners,
  colframe=assumptionred,
  colback=assumptionred!2,
  coltitle=black,
  fonttitle=\bfseries,
  borderline west={2pt}{0pt}{assumptionred},
  before skip=10pt,
  after skip=10pt,
  boxrule=0pt,
  left=10pt,
  right=10pt,
  breakable
}

\tcolorboxenvironment{remark}{
  enhanced jigsaw,
  sharp corners,
  colframe=remarkorange,
  colback=remarkorange!2,
  coltitle=black,
  fonttitle=\bfseries,
  borderline west={2pt}{0pt}{remarkorange},
  before skip=10pt,
  after skip=10pt,
  boxrule=0pt,
  left=10pt,
  right=10pt,
  breakable
}

\newtcolorbox{figurebox}{
  enhanced jigsaw,
  sharp corners,
  colframe=remarkorange,
  colback=remarkorange!2,
  before skip=10pt,
  after skip=10pt,
  boxrule=1pt,
  breakable,
  before skip=10pt,
  after skip=10pt,
}

\newtcolorbox{myproof}{
  enhanced jigsaw,
  sharp corners,
  colframe=schemableu,
  colback=schemableu!0,
  borderline west={1pt}{0pt}{schemableu},
  borderline south={1pt}{0pt}{schemableu},
  before skip=10pt,
  after skip=10pt,
  boxrule=0pt,
  left=10pt,
  right=10pt,
  breakable,
}


%% file: propandtheorem.tex
\newcommand\lemphiequalsB{
    For $\NN$ a DAG ReLU network with pathlifting function $\PathL$, let $|\cdot|$, $\exp$ and $\log$ be the entry-wise application of the associated function, and let $S\in\R(\dimPathL, \dimPathL)$ be the diagonal matrix whose diagonal coefficients are the signs of $\PathL(\theta)$, then $\mu$ almost surely on $\theta$ :{\hfill\color{gray}Proof:\ref{proof:phiequalsB}}
    \begin{align}
    \PathL(\theta) = S\exp\big(\Skeleton\log(|\theta|)\big)\quad \in\R(\dimPathL)\;.\label{eq:PathLequalsSK}
    \end{align}
}

\newcommand{\propPathLJacobian}{
    For $\NN$ a DAG ReLU network with pathlifting function $\PathL$, let $\diag$ be the operator that maps a vector $u$ to the diagonal matrix with diagonal entries given by $u$, then $\mu$ almost surely on $\theta$ : {\hfill\color{gray}Proof:\ref{proof:PathLJacobian}}
    \begin{align*}
        \partial_\theta\PathL &= \diag(\PathL(\theta))\Skeleton\diag(\frac{1}{\theta})
    \end{align*}
    where $\frac{1}{\theta}\in\R(\dimTheta)$ is the vector whose $i^{th}$ coordinate is equal to $\frac{1}{\theta_i}$.  
}

\newcommand{\corskeletonequalsJacone}{
    For $\NN$ a DAG ReLU network with pathlifting function $\PathL$, let $\dimTheta$ be the dimension of the parameter $\theta$ and $\mathbf{1}_\dimTheta\in\R(\dimTheta)$ be the vector full of ones, then
    \begin{align*}
        \Skeleton = \partial_{\mathbf{1}_\dimTheta}\PathL\;.
    \end{align*}
}

\newcommand{\corrkPathLSkeleton}{
     For $\NN$ a DAG ReLU network with pathlifting function $\PathL$, then $\mu$ almost surely on $\theta$, the rank of the Jacobians of $\PathL$ equals to the rank of its skeleton matrix, i.e. $\mu$ almost surely on $\theta$  {\hfill\color{gray}Proof:\ref{proof:rkPathLSkeleton}}
     \begin{align*}
         \rk(\partial_\theta \PathL) = \rk(\Skeleton)\;.
     \end{align*}
}

\newcommand{\thrkSkeleton}{
    For $\NN$ a DAG ReLU network with $E$ and $H$ its sets of edges and hidden nodes, then for $\#$ the cardinality operator, the rank of its skeleton matrix $\Skeleton$ is : {\hfill\color{gray}Proof:\ref{proof:rkSkeleton}}
    \begin{align*}
    \rk(\Skeleton) = \#E - \#H\;.    
    \end{align*}   
}

\newcommand{\proprkdphiequalsrkA}{
    For $\NN$ a DAG ReLU network with $d$ parameters and $h$ hidden nodes then, $\mu$ almost surely on $\theta$ :
     \begin{align*}
         \rk(\partial_\theta \PathL) = \dimTheta - h \;.
     \end{align*}
}

\newcommand{\lemBminusBHzero}{
    For a DAG ReLU network with no hidden nodes, for all $l\geq2$, for all path $\pathn$ of length $l$, let $\edge_1, \ldots \edge_l$ be the ordered list of edges crossed by path $\pathn$, then it exists a path $\tilde{\pathn}$ of size $l-1$ that goes through the list of edges $\edge_1, \ldots, \edge_{l-1}$, as a consequence{\hfill\color{gray}Proof:\ref{proof:lemH0}}
        \begin{align}
            \Skeleton[\pathn, :] - \Skeleton[\tilde{\pathn}, :] &= \begin{matrix}
            \begin{pmatrix} 0 & \cdot& 0 & 1 & 0 &\cdot & 0\end{pmatrix}
            \end{matrix}\label{eq:symH0SK}
        \end{align} where the only non-null element is at the index associated with the edge $\edge_l$, the last edge crossed by path $\pathn$.
}

\newcommand{\lemppcoincideSK}{
    Consider a graph whose node indices are ordered with a topological sort, and let $k$ be an output node such that for all $k' > k$ node $k'$ is an output node, then, for any $\pathn$ of length $l\geq 2$ if $p$ crosses node $k$ and not ending at $k$ 
    there exists a path $\tilde{p}$ of size $l-1$ that coincides with $p$ on its $l-1$ first nodes. As a consequence :{\hfill\color{gray}Proof:\ref{proof:ppcoincide}}
     \begin{align}
        \tilde{\Skeleton}[\pathn, :] - \tilde{\Skeleton}[\tilde{\pathn}, :] &=\begin{matrix}
            \begin{pmatrix} 0 & \cdot& 0 & 1 & 0 &\cdot & 0\end{pmatrix}
            \end{matrix}\label{eq:symcoincideSK2}
    \end{align} where the only non-null element corresponds to the last edge crossed by path $\pathn$.
}
\newcommand{\PhiLinearizesf}{
For a DAG ReLU network with parameters $\theta\in\R(\dimTheta)$ and prediction function $f_\theta$, let $\{x_i\}_{i=1}^n\in\R(\dimx, n)$ be a finite dataset and $\mu$ the Lebesgue measure on $\theta$, then $\mu-$ almost surely, there exists $O$ a neighborhood of $\theta$ and $\mathscr{L}: \R(\dimPathL)\mapsto\R(\dimy, n)$ a linear application such that 
\begin{align*}
    \forall \theta'\in O \quad \{f_{\theta'}(x_i)\}_{i=1}^n =  \mathscr{L}\big(\PathL(\theta')\big) \;.
\end{align*} 
{\hfill\color{gray}Proof:\ref{proof:PhiLinearizesf}}}

%% file: Introduction.tex
\section{Introduction}
Most network architectures can be represented by Directed Acyclic Graph (DAG) ReLU networks, a generalized version of feed-forward ReLU networks. The output of such a network is computed through the computation of its DAG, where each node is the application of the ReLU or the identity function, and each edge is a linear application whose coefficient is given by a parameter of the network.
Compared to arbitrary computational graphs, DAG ReLU Networks are rather simple to study.
This is partly due to the piecewise linearity of those network activation functions, which simplifies the network prediction and has permitted proofs on its approximation properties \cite{ApproxTheorem2Layers,RepresenterTheorem,ApproxTheoremResNet,MLPGenrelizarion}, its generalization and optimization properties \cite{NTK,LeasyTraining}, on its Lipschitz constant \cite{LipchitzBoundwithPath}, and on its intrinsic dynamics \cite{IntrinsicTrainingDynamics}.

Among the different tools that have brought to light the theoretical properties of DAG ReLU networks, we remark the pathlifting function $\PathL$, a function mapping the network parameters to their vector of paths.
The pathlifting is indeed a well-suited tool for the theoretical study of DAG ReLU networks because it locally linearises the network's prediction function and is agnostic to the non-negative homogeneity of the ReLU activation. 
In particular, the pathlifting function has been used to establish generalization bounds \cite{RadmacherAndPathnorm,toolkit}, study the intrinsic dynamics of its pathlifting kernel matrix \cite{IntrinsicTrainingDynamics}, accelerate the network training \cite{PathCond} and penalize the training criterion \cite{RepresenterTheorem,pathSGD}.
Although the pathlifting function has been used in theory, it is rarely used in practice because it is a very high-dimensional vector. 
Generally, only some of its attributes are computed, for example, its dimension and its norm that can be computed in one forward pass \cite{PathLiftingPhD,toolkit}, and the diagonal coefficient of its pathlifting kernel matrix $\trans{\partial_\theta\PathL}{\partial_\theta\PathL}$ that can be computed in one forward-backward pass \cite{PathCond}.

The focus of this paper is the Jacobian of the pathlifting function, $\partial_\theta \PathL$, which has already been studied in  \cite{IntrinsicTrainingDynamics} through the dynamics of the pathlifting kernel matrix $\trans{\partial_\theta\PathL}{\partial_\theta\PathL}$. 
An interesting result of this previous work, which we are able to recover in this paper with a different methodology, is that for a DAG ReLU network with $\dimTheta$ parameters and $h$ hidden neurons, the rank of its pathlifting Jacobian $\partial_\theta\PathL$ is equal to $\dimTheta - h$. 
In the work of \cite{IntrinsicTrainingDynamics} this theorem is proved by an equal upper and lower bound on the rank of the Jacobian matrix $\partial_\theta\PathL$, where the upper bound is provided by a result of the paper (Corollary 3.4) and the lower bound (Appendix F) is obtained using a more general theorem from graph theory \cite{circuit}. 
In fact, this double bounds and the use of an external theorem to conclude on the rank of the Jacobian $\partial_\theta\PathL$ makes the proof pretty technical and convoluted.  

In our work, we take a different approach to determine the rank of the pathlifting Jacobian of a DAG ReLU network. We connect this Jacobian to the network's skeleton matrix (a sparse matrix encoding the paths of the network) and prove its rank by induction on the network number of hidden neurons. Compared to the work of \cite{IntrinsicTrainingDynamics}, our proof is elementary and mainly relies on a topological sort of the graph and some symmetries of the skeleton matrix. 

To prove the theorem on this Jacobian rank, we use intermediate propositions that rewrite the pathlifting and its Jacobian as a matrix product between the parameters, the pathlifting, and the skeleton matrix, which, on top of permitting to conclude on the rank of the pathlifting Jacobian, 
also provides an efficient way to compute the pathlifting Jacobian without backpropagation where the differentiations on the elements of the pathlifting vector are replaced by a cheaper matrix product.
While some of those intermediate results were also used and proved in \cite{IntrinsicTrainingDynamics}, we provide new proofs based on first-order expansion and a Python implementation available at this \href{https://gitlab.inria.fr/mverbock/pathliftingandskeleton/-/tree/main?ref_type=heads}{GitLab}.

The paper is organized as follows. In the first section, we define the main objects of the paper and state the main propositions and the main theorem. 
In a second section, we provide a sketch of proof for the main theorem. 
In the last section, we quantify experimentally the gain of using the proposed theory to compute the pathlifting Jacobian compared to backpropagation.  

\subsection{General notations and objects}\label{sec:notations}
We note $\R(\dimx)$ the $\dimx$-dimensional real vector space and $\R(\dimx_1, \cdots, \dimx_q)$ the $(\dimx_1, \cdots, \dimx_q)$-dimensional real vector space. 

Let $f : \R(\dimTheta) \mapsto \R(\dimy)$ be a function, when $f$ is differentiable, we note $\partial_x f\in \R(\dimy, \dimTheta)$ the Jacobian matrix of $f$ at $x\in \R(\dimTheta)$ and when $f$ is a real-valued function, we note $\nabla_x f\in \R(\dimTheta)$ its gradient vector with the convention that $\nabla_x f = \trans{\left(\partial_x f\right)}$. 

We represent a Directed Acyclic Graph (DAG) $\Graph$ by its list of nodes $\Nodes = \{n_1, \cdots n_\dimNode\}$ and its list of oriented edges $\Edges = \{e_1, \cdots e_\dimTheta\}$ with the property that there is no cycle in the graph.

We define a DAG ReLU network $\NN$, a computational graph whose graph is a DAG where some nodes are designated as input and output nodes. This network is also characterized by its parameters $\theta\in\R(\dimTheta)$ and its prediction function noted $f_\theta$. 

For a given DAG ReLU network $\NN$, let $\dimPathL$ be the number of paths connecting one input to one output node in its DAG, we note $\PathL(\theta)\in\R(\dimPathL)$ the vector whose value at index $p\in \{1, \cdots, \dimPathL\}$ is the product of the network weights along the path $p$.

The dimensions of the main objects are summarized in the following table.
\begin{table}[h!]
    \begin{center}
        \caption{Objects dimensions.}
        \begin{tabular}{|c|c|l|}
            \hline
            Symbol  & class & Description \\
            \hline
            $\dimx, \dimy$& $\mathbb{N}^*$ & number of input and output nodes \\
            $\theta $& $ \R(\dimTheta)$ & Parameters of the network\\
            $\PathL(\theta)$&$\R(\dimPathL)$ & Pathlifting of the network\\
            $\partial_\theta \PathL$&$\R(\dimPathL, \dimTheta)$ & Jacobian of the pathlifting\\
            \hline
        \end{tabular}
        \label{tab:objects_dimensions}
    \end{center}
\end{table}

%% file: Theorems.tex
\section{Properties of the pathlifting and its Jacobian}
In this section, we present the DAG ReLU network, the pathlifting function, and define the skeleton matrix. Then we connect those objects to one another through propositions.
\subsection{Main definitions}
We consider a \textbf{DAG ReLU network} $\NN$ with vector of parameters $\theta\in\R(\dimTheta)$ and with DAG $\Graph$ described by its list of nodes $\Nodes = \{n_1, \ldots, n_\dimNode\}$ and its list of oriented edges $\Edges = \{e_1, \ldots, e_{\dimTheta}\}$. The indices of the edges in $\Edges$ are isomorphic to the indices of $\theta$ and the prediction function of $\NN$, noted $f_\theta:\R(\dimx)\mapsto\R(\dimy)$,  is defined as the result of the computational graph of $\Graph$ where each node represents the application of the ReLU $\max(0, \cdot)$ or the identity function, and each edge is the linear application whose coefficient is the parameter associated to it.

The definition and the association of the network $\NN$ with its DAG $\Graph$ allow us to distinguish three subsets of nodes : the input nodes $I$, the output nodes $O$, and the hidden nodes $H$. The input nodes are associated with the network's input and can be identified in the graph as the nodes with no incoming edge. The output nodes are associated with the network's outputs and cannot be inferred from the graph representation, they
contains every node with no outgoing edge together with any additional node designated as an output by the user and that are not an input node. Finally, a node is a hidden node when it is neither an input nor an output node. Remark that those three subsets are a partition of $\Nodes$, i.e 
\begin{align}
    I\cap H = I\cap O = H\cap O = \emptyset &&\text{and} &&I\cup H\cup O = N \;.
\end{align}
\begin{remark}
 With this convention, the network's biases are treated as input nodes, and their input values are set to $1$.
\end{remark}
The \textbf{pathlifting function} $\PathL$ of a network is a function that maps its parameters $\theta$ to its vectors of paths, where a path is a set of edges that connects an input node to an output node in the graph. For a DAG ReLU network with DAG $\Graph$, let $\dimPathL$ be the number of paths connecting one input node to one output node in the graph $\Graph$ then, the pathlifting function $\PathL:\R(\dimTheta)\mapsto\R(\dimPathL)$ is the function whose value at index $j$ for $j\in\{1, \cdots, \dimPathL\}$ equals the product of all the parameters along the path $j$. Strictly speaking, we define the pathlifting function as follows.
\begin{definition}
 For a DAG ReLU network with DAG $\Graph$ and with an edge list $\Edges$ isomorphic to the index of the network parameters $\theta$, let $\dimPathL$ be the number of paths in the network then, for $j\in\{1, \ldots, \dimPathL\}$ a path index, let $e_1, \ldots, e_{d_j}$ be the list of edges crossed by the path $j$, then for $\theta\in\R(\dimTheta)$, we define $\PathL(\theta)$ at index $j$ as :
    \begin{align}
 \PathL(\theta)_j = \prod_{k=1}^{d_j}\theta_{e_k}
    \end{align}
\end{definition}
The main characteristic of the pathlifting function is that, for a finite-dimensional dataset, it locally linearises the network's prediction function. While this proposition is well-known in the literature (\cite{toolkit,NNReLuIdentifiability}), we recall this property in the next proposition for the sake of completeness. 
\begin{proposition}\label{prop:PhiLinearizesf}
    \PhiLinearizesf    
\end{proposition}
The function $\mathscr{L}$ is provided in the proof of \cref{prop:PhiLinearizesf} and is encoded in the \href{https://gitlab.inria.fr/mverbock/pathliftingandskeleton/-/tree/main?ref_type=heads}{Python module}.
For more details on the definitions and the properties of the pathlifting function, we refer the reader to Chapter 2 of \cite{PathLiftingPhD}.

The last object we introduce is the \textbf{skeleton matrix} $\Skeleton$ of a network. The skeleton matrix of a DAG ReLU network is a sparse matrix encoding the absolute dependency of the network's paths with respect to its parameters. Strictly speaking, we define the skeleton matrix as follows.
\begin{definition}
 Let $\NN$ be a DAG ReLU network with $\dimTheta$ parameters and with $\dimPathL$ paths, we define $\Skeleton\in\R(\dimPathL, \dimTheta)$, the skeleton matrix of $\NN$ as the matrix whose coefficients are defined as follows.
    \begin{align*}
 \Skeleton_{i j} =
         \begin{cases}
          1 \quad\text{if the $i^{th}$ path of $\NN$ depends on the parameters $\theta_j$}  \\
          0 \quad \text{otherwise}
         \end{cases}
    \end{align*} 
\end{definition}
To finish this subsection, we provide in \cref{fig:DAG_intro} an example of a DAG ReLU network with its pathlifting function and its skeleton matrix.
\begin{figurebox}
    \begin{minipage}{0.44\textwidth}
            \centering
            \includegraphics[width=0.9\textwidth]{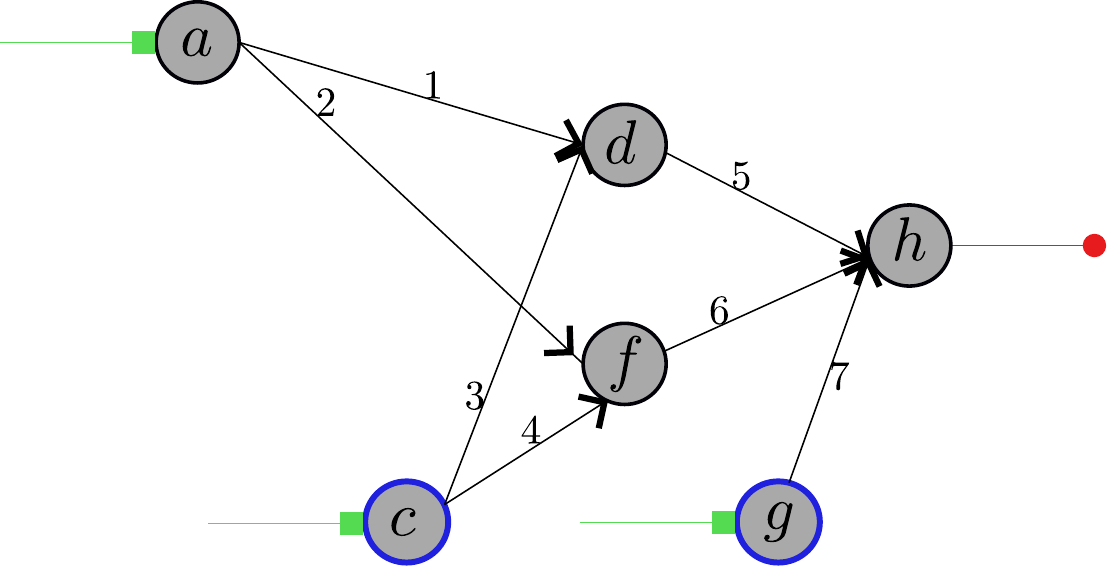}
    \end{minipage}
    \begin{minipage}{0.55\textwidth}
        \begin{align*}
        \PathL(\theta) &= \begin{pmatrix}
            \theta_1\theta_5\\
            \theta_2\theta_6\\
            \theta_3\theta_5\\
            \theta_4\theta_6\\
            \theta_7
        \end{pmatrix}
        &&\Skeleton = \begin{pmatrix}
            1 & 0 & 0 & 0 & 1 & 0 & 0\\
            0 & 1 & 0 & 0 & 0 & 1 & 0\\
            0 & 0 & 1 & 0 & 1 & 0 & 0\\
            0 & 0 & 0 & 1 & 0 & 1 & 0\\
            0 & 0 & 0 & 0 & 0 & 0 & 1
        \end{pmatrix} 
    \end{align*}
    \end{minipage}
    \captionof{figure}{left: DAG  of a feed-forward ReLU network with bias and one hidden layer. The nodes are represented with circles and the edges with black oriented arrows. The green (resp. red) arrows indicate the input (resp. the output) nodes, and the bias nodes are circled in dark blue. The nodes are indexed with letters, while the edges are indexed with integers. The node sets are: $I = \{a, c, g\}, O = \{h\}, H = \{d, f\}$. Middle: pathlifting function $\PathL$ of the network whose DAG is the left figure. Right: skeleton matrix $\Skeleton$ associated to $\PathL$ of the middle figure.}
    \label{fig:DAG_intro}
\end{figurebox}

\subsection{Propositions and theorem}
We now present propositions that connect the pathlifting function, its Jacobian, and the skeleton matrix, and then present the main theorem on the rank of the pathlifting Jacobian. We note to the reader that some of the propositions were already known in the literature (\cite{IntrinsicTrainingDynamics}), in particular \cref{prop:PathLJacobian,th:rkSkeleton}. However, we provide a new proof of \cref{prop:PathLJacobian} using a first-order Taylor expansion, and a new proof of \cref{th:rkSkeleton} by induction.  

Let $\mu$ be a measure on $\theta$ that is absolutely continuous with respect to the Lebesgue measure, the following equality holds. 
\begin{lemma}\label{lem:phiequalsB}
    \lemphiequalsB
\end{lemma} 
This result is a key element that is to be used to compute the pathlifting's Jacobian using the skeleton matrix.
\begin{proposition}\label{prop:PathLJacobian}
    \propPathLJacobian
\end{proposition}
In fact, for both results, the $\mu$ almost surely statement excludes the ill-conditioned cases where $\theta$ has a null coordinate.

We remark that \cref{prop:PathLJacobian} provides an expression of the pathlifting Jacobian as a function of $\PathL$ and $\theta$, and whose total computational complexity is $O(\dimTheta\dimPathL)$ as each element $(i, j)$ of $\partial_\theta\PathL$ can be computed independently with $\PathL_iB_{i, j}\frac{1}{\theta_j}$. In fact, this asymptotic computational cost is similar to the one obtained when backpropagating each element of $\PathL$; however, as shown in the experimental section of this paper, the matrix product in \cref{prop:PathLJacobian} is naturally parallelizable and is super-efficient in practice compared with standard backpropagation.


From \cref{prop:PathLJacobian}, one can deduce that the skeleton matrix $\Skeleton \in\R(\dimPathL, \dimTheta)$ is equals to the jacobian of $\PathL$ at point $\theta = \mathbf{1}_\dimTheta\in\R(\dimTheta)$ the vector full of ones. 
\begin{corollary}
    \corskeletonequalsJacone
\end{corollary} 
From \cref{prop:PathLJacobian}, one can easily deduce an equality on the rank of the pathlifting Jacobian.
\begin{corollary}\label{cor:rkPathLSkeleton}
     \corrkPathLSkeleton
\end{corollary}
We conclude this section with the main theorem that is provided with a sketch of proof in \cref{sec:sketch}. 
\begin{theorem}\label{th:rkSkeleton}
    \thrkSkeleton
\end{theorem}
\Cref{cor:rkPathLSkeleton,th:rkSkeleton} lead to the following corollary, which concludes the section. 
\begin{corollary}\label{prop:rkdphi=rkA}
     \proprkdphiequalsrkA
\end{corollary}

%% file: SketchOfProof.tex
\section{Sketch of proof}\label{sec:sketch}
In this section, we provide a sketch of proof for \cref{th:rkSkeleton}. To mimic the induction of the general proof, we chose two simple DAG ReLU networks on which we apply the general proof.
The first network is chosen with no hidden nodes to illustrate the initialization of the induction, and the second network is chosen with one hidden node to illustrate one step of the induction.
\subsection{Initialization of the induction}
We consider a DAG ReLU network with no hidden nodes, $h=0$, and whose DAG is given in \cref{fig:NoHiddenNodeSketchOfProof}. Its pathlifting function $\PathL$ and skeleton matrix $\Skeleton$ are defined in  \cref{eq:PhiAndSkeletonNoHiddenNodeSketchOfProof}. The columns of the skeleton matrix $\Skeleton$ are ordered according to the list of edges indicated at its top, and its rows are ordered to match the pathlifting function indexing. 
\begin{figurebox}
    \begin{minipage}{0.44\textwidth}
        \centering
        \includegraphics[width = 0.9\textwidth]{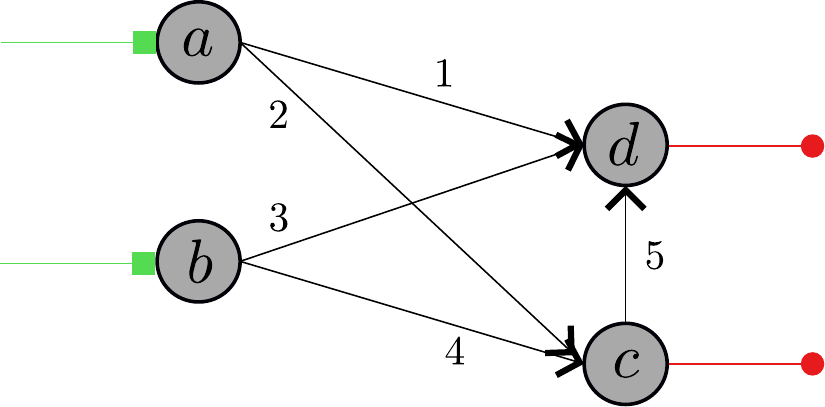}
        \captionof{figure}{DAG of a network with no hidden node, the green and the red arrows represent the input nodes and the output nodes, respectively. The nodes are indexed with letters $\Nodes = \{a, b, c, d\}$ with $I=\{a, b\}$ and $O=\{c, d\}$, and the edges with integers $\Edges = \{1, 2, 3, 4, 5\}$. 
    }
        \label{fig:NoHiddenNodeSketchOfProof}
    \end{minipage}
    \hfill
    \begin{minipage}{0.555\textwidth}
        \begin{align}
            \PathL(\theta) = 
            \begin{matrix}
                \begin{matrix}
                    \\
                \end{matrix}\\
                \begin{pmatrix}
                    \theta_1 \\
                    \theta_2 \\
                    \theta_2\theta_5\\
                    \theta_3\\
                    \theta_4\\
                    \theta_4\theta_5    
                \end{pmatrix}
            \end{matrix},\;
            \Skeleton = 
            \begin{matrix}
                    1 \quad 2 \quad 3 \quad 4 \quad 5\\
                \begin{pmatrix}
                    1 & 0 & 0 & 0 & 0\\
                    0 & 1 & 0 & 0 & 0\\
                    0 & 1 & 0 & 0 & 1\\
                    0 & 0 & 1 & 0 & 0\\
                    0 & 0 & 0 & 1 & 0\\
                    0 & 0 & 0 & 1 & 1    
                \end{pmatrix} 
            \end{matrix}\label{eq:PhiAndSkeletonNoHiddenNodeSketchOfProof}
        \end{align}
    \end{minipage}\\
\end{figurebox}

To compute the rank of $\Skeleton$, we apply invertible row operations on it, then evaluate the rank of the resulting matrix. Indeed, performing row operations on a matrix is equivalent to applying invertible matrices to its left side; thus, the rank is not changed.

We choose the row operations that take advantage of some specific properties of the skeleton matrix $\Skeleton$ and which are embodied by the following lemma.
\begin{lemma}\label{lem:H0SK}
        \lemBminusBHzero
\end{lemma}
This lemma can be easily checked on the \cref{fig:NoHiddenNodeSketchOfProof}: the only paths of length greater than one are the third and sixth path of $\PathL$, they respectively go across the list of edges $2, 5 $ and $4, 5$ and indeed the second path of $\PathL$ is of length one and goes through $2$ and the fifth path of $\PathL$ is of length one and goes through $4$.  

Using this symmetry, we construct \cref{algo:rowoperationSkeletonSK}, which takes as input the matrix $\Skeleton$ and outputs a matrix $\tilde{\Skeleton}$ of the same rank as $\Skeleton$. 

\begin{minipage}{0.65\textwidth}
\noindent\rule[7pt]{\linewidth}{0.4pt}
    \normalem 
    \begin{algorithm}[H]
     \KwData{$\Skeleton$ the skeleton matrix, $l_{max}$ the maximum length of a path}
     \KwResult{matrix $\tilde{\Skeleton}$ of the same rank as $\Skeleton$}
     \For{$l = l_{max}, l_{max} - 1, \ldots, 2$}{
      \For{$\pathn$ of length $l$}{
      let $\tilde{\pathn}$ be the path defined in \cref{lem:H0SK} for path $\pathn$\;
      $\Skeleton[\pathn, :] \xleftarrow{} \Skeleton[\pathn, :] - \Skeleton[\tilde{\pathn}, :]$\;
      }
     }
     \caption{Row operations on matrix $\Skeleton$}
     \label{algo:rowoperationSkeletonSK}
    \end{algorithm}
\noindent\rule[7pt]{\linewidth}{0.4pt}
\end{minipage}
\hfill
\begin{minipage}{0.33\textwidth}
    \begin{align}
        \hspace{-0.3cm}\tilde{\Skeleton} = \begin{pmatrix}
            1 & 0 & 0 & 0 & 0\\
            0 & 1 & 0 & 0 & 0\\
            0 & 0 & 0 & 0 & 1\\
            0 & 0 & 1 & 0 & 0\\
            0 & 0 & 0 & 1 & 0 \\
            0 & 0 & 0 & 0 & 1
        \end{pmatrix}\label{eq:tildeBH0SK}
    \end{align}
\end{minipage}
Let $\tilde{\Skeleton}$ be the matrix outputted by $\cref{algo:rowoperationSkeletonSK}$ and provided in \cref{eq:tildeBH0SK}, then, because of \cref{lem:H0SK}, each of its rows has only one non-null element thus its rank is equal to the number of its non-null columns, which is equal to $5$, which is also the dimension of $\theta$. It follows that  
\begin{align}
    \rk(\Skeleton) = \rk(\tilde{\Skeleton}) = 5 = \dimTheta - 0\;.
\end{align}
\subsection{One induction step}
We now suppose that the theorem is true for any DAG ReLU network without hidden nodes and perform one step of induction. We choose a network $\NN$ with one hidden node $h=1$ and whose representation is given in  \cref{fig:1HiddenNodeSketchOfProof}.
Its pathlifting $\PathL$ and its skeleton matrix $\Skeleton$ are described in \cref{eq:PhiAndSkeleton1HiddenNodeSketchOfProof1,eq:PhiAndSkeleton1HiddenNodeSketchOfProof2}.
\begin{figurebox}
    
    \begin{minipage}{0.6\textwidth}
        \centering
        \includegraphics[width = 0.6\textwidth]{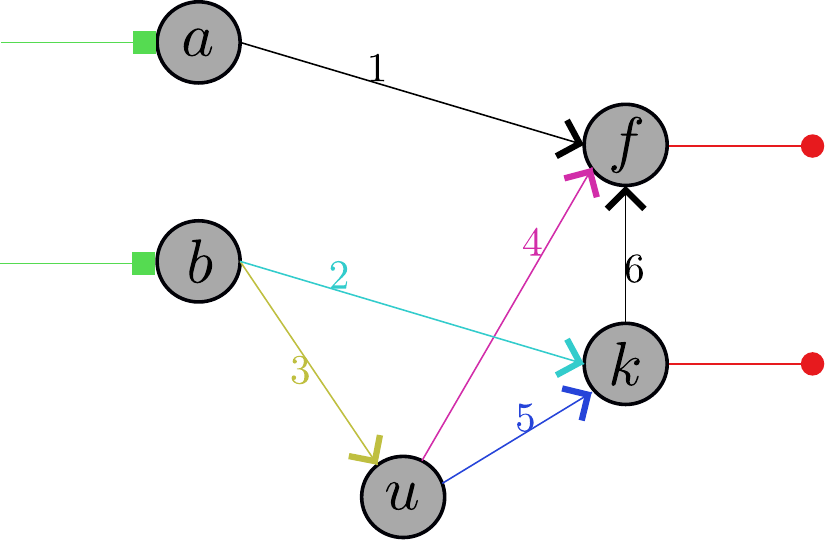}
        \captionof{figure}{DAG of a network with one hidden node, the green and the red arrows represent the input nodes and the output nodes, respectively. The nodes are indexed with letters and the edges with integers. The set of edges $\textcolor{colorEum}{E_{-}^u = \{3\}}$, $\textcolor{colorEup}{E_{+}^u = \{4\}}$ and $\textcolor{colorEkm}{E_{-}^k = \{2\}}$ are colored in dark yellow, purple and cyan. The blue edge $\textcolor{colore}{5}$ is the edge connecting the hidden node $u$ to the output node $k$.}
        \label{fig:1HiddenNodeSketchOfProof}
        
    \end{minipage}
    \hfill
    \begin{minipage}{0.4\textwidth}
        \begin{align}
            &\PathL(\theta) = \begin{pmatrix}
                \theta_1 \\
                \theta_2\\
                \theta_2\theta_6\\
                \theta_3\theta_4 \\
                \theta_3\theta_5\\
                \theta_3\theta_5\theta_6   
            \end{pmatrix} \label{eq:PhiAndSkeleton1HiddenNodeSketchOfProof1}\\
            \nonumber\\
            &\Skeleton = \begin{pmatrix}
                1 & 0 & 0 & 0 & 0 & 0\\
                0 & 1 & 0 & 0 & 0 & 0\\
                0 & 1 & 0 & 0 & 0 & 1\\
                0 & 0 & 1 & 1 & 0 & 0\\
                0 & 0 & 1 & 0 & 1 & 0 \\
                0 & 0 & 1 & 0 & 1 & 1
            \end{pmatrix} \label{eq:PhiAndSkeleton1HiddenNodeSketchOfProof2}
        \end{align}
    \end{minipage}
\end{figurebox}
For this part of the sketch, we need a topological sort of the network's nodes, which we choose as the list $\{a, b, u, k, f\}$. The key idea of the induction is \uline{to transform the last hidden node of the topological sort (node $u$), into an output node and then use the induction hypothesis}. 
In fact, this goal is achievable because the column of the skeleton matrix associated to the edge $\textcolor{blue}{5}$ connecting node $u$ and node $k$ is linearly dependent on the columns associated with the incoming and outcoming edges of node $u$.

We perform this operation in two steps, which is one extra step compared to the induction initialization. In a first step, we permute the columns and the rows of $\Skeleton$ to uncover its structure then, we perform column and row operations on it. 
\begin{myproof}
\textbf{step 1 : re-ordering}\\
We consider $u$ the unique hidden node of the network and identify the next node in the topological sort that is connected to it: node $k$. \\
Then, we define four types of edges in the graph, which can be visualized with different colors in \cref{fig:1HiddenNodeSketchOfProof}. The edge $\textcolor{blue}{5}$, that make the connection between node $u$ and node $k$; the out-going edges of $u$ except $\textcolor{blue}{5}$ noted $\textcolor{colorEup}{E_{+}^u}$; the incoming edges of node $u$ noted $\textcolor{colorEum}{E_{-}^u}$; and the in-going edges of node $k$ except $\textcolor{blue}{5}$ noted $\textcolor{colorEkm}{E_{-}^k}$. We have
\begin{align*}
&&\textcolor{colorEup}{E_{+}^u = \{4\}} &&\textcolor{colorEum}{E_{-}^u = \{3\}} &&\textcolor{colorEkm}{E_{-}^k = \{2\}} \;.
\end{align*}
Then, we rearrange the columns of $\Skeleton$ such that its first column is associated to the edge $\textcolor{blue}{5}$, the second column is associated to the edge in $\textcolor{colorEup}{E_{+}^u}$, the third column to the edge in $\textcolor{colorEum}{E_{-}^u}$, the fourth column to the edges in $\textcolor{colorEkm}{E_{-}^k}$ and then the rest of the edges. We save the result of this permutation in the matrix $\Skeleton'$ whose columns are now ordered with respect to the list of edges $\{\textcolor{blue}{5}, \textcolor{colorEup}{4}, \textcolor{colorEum}{3}, \textcolor{colorEkm}{2}, 1, 6\}$. The rank of matrix $\Skeleton'$ is equal to the rank of matrix $\Skeleton$.

\begin{minipage}{0.48\textwidth}
    \begin{align}
    \Skeleton' = \begin{matrix}
                    \; \textcolor{colorEup}{E_+^u}\; \textcolor{colorEum}{E_{-}^u}\; \textcolor{colorEkm}{E_{-}^k}\;\; \;\;\\
                    \textcolor{blue}{5}\quad \textcolor{colorEup}{4}\quad \textcolor{colorEum}{3}\quad \textcolor{colorEkm}{2}\quad 1\quad 6\\
                    \vspace{-0.2cm}\\
                    \begin{pmatrix}
                        0 & 0 & 0 & 0 & 1 & 0\\
                        0 & 0 & 0 & 1 & 0 & 0\\
                        0 & 0 & 0 & 1 & 0 & 1\\
                        0 & 1 & 1 & 0 & 0 & 0\\
                        1 & 0 & 1 & 0 & 0 & 0 \\
                        1 & 0 & 1 & 0 & 0 & 1
                    \end{pmatrix} 
                \end{matrix}\label{eq:SkeletonPrime1HiddenNodeSketchOfProof}
    \end{align}
\end{minipage}
\hfill
\begin{minipage}{0.5\textwidth}
    \begin{align}
    \Skeleton'' = 
            \begin{matrix}
                \begin{matrix}
                    \\
                    \\
                    \vspace{-0.1cm}\\
                    \PathL(\theta)_5\\
                    \PathL(\theta)_6\\
                    \PathL(\theta)_4\\
                    \PathL(\theta)_2\\
                    \PathL(\theta)_3\\
                    \PathL(\theta)_1
                \end{matrix}&
                \begin{matrix}
               \; \textcolor{colorEup}{E_+^u}\; \textcolor{colorEum}{E_{-}^u}\; \textcolor{colorEkm}{E_{-}^k}\;\; \;\;\\
                    \textcolor{blue}{5}\quad \textcolor{colorEup}{4}\quad \textcolor{colorEum}{3}\quad \textcolor{colorEkm}{2}\quad 1\quad 6\\
                    \vspace{-0.2cm}\\
                    \begin{pmatrix}
                        1 & {\color{blue}{0}} & 1 & {\color{purple}{0}} & 0 & 0 \\
                        1 & {\color{blue}{0}} & 1 & {\color{purple}{0}} & 0 & 1 \\
                        {\color{blue}{0}} & 1 & 1 & {\color{purple}{0}} & 0 & 0\\
                        {\color{teal}{0}} & {\color{teal}{0}} & {\color{teal}{0}} & 1 & 0 & 0\\
                        {\color{teal}{0}} & {\color{teal}{0}} & {\color{teal}{0}} & 1 & 0 & 1\\
                        {\color{teal}{0}} & {\color{teal}{0}} & {\color{teal}{0}} & 0 & 1 & 0
                    \end{pmatrix}
                \end{matrix} 
            \end{matrix}\label{eq:SkeletonDoublePrime1HiddenNodeSketchOfProof}
\end{align}   
\end{minipage}
\vspace{0.2cm}\\
Then, we permute the rows of $\Skeleton'$ by grouping on top the paths crossing both node $u$ and node $k$ : $\{\PathL(\theta)_5, \PathL(\theta)_6\}$, then all the paths crossing only node $u$ : $\{\PathL(\theta)_4\}$, then all the paths crossing only node $k$ : $\{\PathL(\theta)_2, \PathL(\theta)_3\}$ and then all the  remaining ones :$\{\PathL(\theta)_1\}$. We save the result of this permutation in the matrix $\Skeleton''$ whose rows are now ordered with respect to the list of paths $\{\PathL(\theta)_5, \PathL(\theta)_6, \PathL(\theta)_4, \PathL(\theta)_2, \PathL(\theta)_3, \PathL(\theta)_1\}$. The rank of matrix $\Skeleton''$ is still equal to rank of matrix $\Skeleton$.

With the writing of $\Skeleton''$, the structure of the initial matrix $\Skeleton$ is revealed with the blue, the green, and the purple zeros, which are proved in the main proof to be the result of the performed row and column re-ordering. 
\end{myproof}
\begin{myproof}
\textbf{step 2: column and row operations}\\
We now perform two operations on the columns of $\Skeleton''$. First, we use the columns associated with the edge in $E_{-}^u$ that is full of ones on its top and full of zeros on its bottom (a property that is proved in the main proof as a direct consequence of the previous permutations). 
We subtract this column to the first column of $\Skeleton''$. Then, we take the column of $E_{+}^u$, i.e., the second column, and we add it to the first column. The resulting matrix $\tilde{\Skeleton}$ is of the same rank as $\Skeleton$ and is equal to 
\begin{align}
    \tilde{\Skeleton}\; = \;
            \begin{matrix}
                \begin{matrix}
                    \\
                    \\
                    \vspace{-0.1cm}\\
                    \PathL(\theta)_5\\
                    \PathL(\theta)_6\\
                    \PathL(\theta)_4\\
                    \PathL(\theta)_2\\
                    \PathL(\theta)_3\\
                    \PathL(\theta)_1
                \end{matrix}&
                \begin{matrix}
                    \; \textcolor{colorEup}{E_+^u}\; \textcolor{colorEum}{E_{-}^u}\; \textcolor{colorEkm}{E_{-}^k}\;\; \;\;\\
                    \textcolor{blue}{5}\quad \textcolor{colorEup}{4}\quad \textcolor{colorEum}{3}\quad \textcolor{colorEkm}{2}\quad 1\quad 6\\
                    \vspace{-0.2cm}\\
                    \begin{pmatrix}
                        0 & {\color{blue}{0}} & 1 & {\color{purple}{0}} & 0 & 0 \\
                        0 & {\color{blue}{0}} & 1 & {\color{purple}{0}} & 0 & 1 \\
                        {\color{blue}{0}} & 1 & 1 & {\color{purple}{0}} & 0 & 0\\
                        {\color{teal}{0}} & {\color{teal}{0}} & {\color{teal}{0}} & 1 & 0 & 0\\
                        {\color{teal}{0}} & {\color{teal}{0}} & {\color{teal}{0}} & 1 & 0 & 1\\
                        {\color{teal}{0}} & {\color{teal}{0}} & {\color{teal}{0}} & 0 & 1 & 0
                    \end{pmatrix}
                \end{matrix} 
            \end{matrix} \; .\label{eq:SkeletonTilde1HiddenNodeSketchOfProof}
\end{align}
We remark that the column associated to the edge $\textcolor{blue}{5}$ is full of zeros, thus the obtained matrix $\tilde{\Skeleton}$ getd closer to the skeleton matrix of a network whose graph is the same of $\NN$ but without edge $\textcolor{blue}{5}$.
In fact, we are one step away from this affirmation as we still need to remove the $1$ associated to the paths that were crossing the edge $\textcolor{blue}{5}$ but not ending at node $k$. Indeed, remark that the second row on $\tilde{\Skeleton}$, and which was previously associated to $\PathL(\theta)_6$, does not correspond to a path anymore and one should remove either its dependency on $6$ or its dependency in $3$.\\

To do so, we perform row operations on $\tilde{\Skeleton}$ using some symmetries of $\tilde{\Skeleton}$ which are similar to those used in the initialization.
\begin{lemma}\label{lem:ppcoincideSK}
    \lemppcoincideSK
\end{lemma}
Once again, this lemma can be easily verified here. The paths crossing node $k$ and not ending at $k$ are the path $\PathL(\theta)_3 = \theta_2\theta_6$ which coincides with $\PathL(\theta)_2 = \theta_2$ and the path $\PathL(\theta)_6 = \theta_3\theta_5\theta_6$ which coincides with $\PathL(\theta)_5 = \theta_3\theta_5$.
We now use the following algorithm to take advantage of \cref{lem:ppcoincideSK}.\\

\begin{minipage}{0.6\textwidth}
\noindent\rule[7pt]{\linewidth}{0.4pt}
    \normalem 
    \begin{algorithm}[H]
     \KwData{$\tilde{\Skeleton}$ the modified skeleton matrix, $\Paths$ the list of path crossing node $k$ and not ending at node $k$, $l_{min}, l_{max}$ the minimum and maximum length of a path in $\Paths$.}
     \KwResult{matrix $\tilde{\Skeleton}'$ of the same rank as $\tilde{\Skeleton}$}
     \For{$l\in \{l_{max}, l_{max} -1,  \cdots, l_{min}\}$}{
     \For{$p\in \Paths$ and $p$ of length $l$}{
      Let $\tilde{\pathn}$ be the  path of \cref{lem:ppcoincideSK} coinciding with $\pathn$ on its $l-1$ first nodes.
      $\tilde{\Skeleton}[\pathn, :]\leftarrow \tilde{\Skeleton}[\pathn, :] - \tilde{\Skeleton}[\tilde{\pathn}, :] $
      }
      }
     \caption{Row operations on matrix $\tilde{\Skeleton}$}
     \label{algo:rowoperationSkeletonReccurenceSK}
    \end{algorithm}
\noindent\rule[7pt]{\linewidth}{0.4pt}
\end{minipage}
\begin{minipage}{0.4\textwidth}
    \begin{align}
        \tilde{\Skeleton}' &= 
                \begin{matrix}
                    \begin{matrix}
                        \; \textcolor{colorEup}{E_+^u}\; \textcolor{colorEum}{E_{-}^u}\; \textcolor{colorEkm}{E_{-}^k}\;\; \;\;\\
                    \textcolor{blue}{5}\quad \textcolor{colorEup}{4}\quad \textcolor{colorEum}{3}\quad \textcolor{colorEkm}{2}\quad 1\quad 6\\
                    \vspace{-0.2cm}\\
                        \begin{pmatrix}
                            0 & {\color{blue}{0}} & 1 & {\color{purple}{0}} & 0 & 0 \\
                            0 & {\color{blue}{0}} & 0 & {\color{purple}{0}} & 0 & 1 \\
                            {\color{blue}{0}} & 1 & 1 & {\color{purple}{0}} & 0 & 0\\
                            {\color{teal}{0}} & {\color{teal}{0}} & {\color{teal}{0}} & 1 & 0 & 0\\
                            {\color{teal}{0}} & {\color{teal}{0}} & {\color{teal}{0}} & 0 & 0 & 1\\
                            {\color{teal}{0}} & {\color{teal}{0}} & {\color{teal}{0}} & 0 & 1 & 0
                        \end{pmatrix}
                    \end{matrix} 
                \end{matrix}\label{eq:tildeSkeletonPrime1HiddenNode0SP}
    \end{align}
\end{minipage}
For our specific case, this algorithm subtracts the row associated with $\PathL(\theta)_2$ from the row associated with $\PathL(\theta)_3$ and the row associated with $\PathL(\theta)_5$ from the row associated with $\PathL(\theta)_6$. We obtain the matrix $\tilde{\Skeleton}'$ that is of the same rank as $\Skeleton$ and is given in \cref{eq:tildeSkeletonPrime1HiddenNode0SP}. 

With the matrix $\tilde{\Skeleton}'$, for any edges connecting two output nodes and which were previously on a path crossing node $u$ (here the edge $6$), it exists a line in $\tilde{\Skeleton}'$ with only one nonzero element on its row associated to that edge index (line $2$ and $5$).

We use these rows, for example the second line of $\tilde{\Skeleton}'$, to remove the dependence on edge $6$ in the other lines of $\tilde{\Skeleton}'$. This operation is done by subtracting the second line of $\tilde{\Skeleton}'$ from the fifth line of $\tilde{\Skeleton}'$. We save the result of this operation in the matrix $\tilde{\Skeleton}''$, which is still of the same rank as $\Skeleton$ and is provided in \cref{eq:tildeSkeletonPrimePrime}.\\
\begin{figurebox}
\begin{minipage}{0.56\textwidth}   
    \begin{align}
        \hspace{-0.7cm}\tilde{\Skeleton}'' = 
                \begin{matrix}
                    \begin{matrix}
                        \hspace{0.4cm}5\quad 4\quad 3\quad 2\quad 1\;\;\;\hspace{-0.25em}\rule[-1.2ex]{0.5pt}{2.8ex}\hspace{0.25em}\hspace{-0.185cm}6\quad\\
                        \begin{pmatrix}
                            0 & {\color{blue}{0}} & 1 & {\color{purple}{0}} & 0 & \hspace{-0.25em}\rule[-1.2ex]{0.5pt}{2.8ex}\hspace{0.25em}\hspace{-0.185cm}0 \\
                            0 & {\color{blue}{0}} & 0 & {\color{purple}{0}} & 0 & \hspace{-0.25em}\rule[-1.2ex]{0.5pt}{2.8ex}\hspace{0.25em}\hspace{-0.185cm}1 \\
                            {\color{blue}{0}} & 1 & 1 & {\color{purple}{0}} & 0 & \hspace{-0.25em}\rule[-1.2ex]{0.5pt}{2.8ex}\hspace{0.25em}\hspace{-0.185cm}0\\
                            {\color{teal}{0}} & {\color{teal}{0}} & {\color{teal}{0}} & 1 & 0 & \hspace{-0.25em}\rule[-1.2ex]{0.5pt}{2.8ex}\hspace{0.25em}\hspace{-0.185cm}0\\
                            {\color{teal}{0}} & {\color{teal}{0}} & {\color{teal}{0}} & 0 & 0 & \hspace{-0.25em}\rule[-1.2ex]{0.5pt}{2.8ex}\hspace{0.25em}\hspace{-0.185cm}0\\
                            {\color{teal}{0}} & {\color{teal}{0}} & {\color{teal}{0}} & 0 & 1 & \hspace{-0.25em}\rule[-1.2ex]{0.5pt}{2.8ex}\hspace{0.25em}\hspace{-0.185cm}0
                        \end{pmatrix}
                    \end{matrix} 
                \end{matrix},\;
            \PathL = \begin{pmatrix} 
                            \theta_3\\
                            \\
                            \theta_3\theta_4\\
                            \theta_2\\
                            \\
                            \theta_1\\
                            \end{pmatrix}\label{eq:tildeSkeletonPrimePrime}
    \end{align}
\end{minipage}
\hfill
\begin{minipage}{0.4\textwidth}
        \includegraphics[width = 0.8\textwidth]{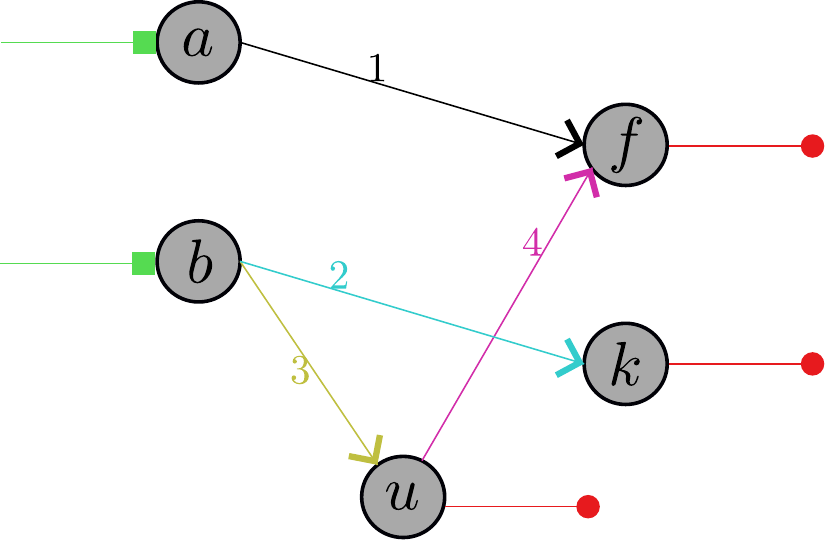}
        \captionof{figure}{
DAG of a network with no hidden node associated with the pathlifting function and skeleton matrix of \cref{eq:tildeSkeletonPrimePrime}. }
        \label{fig:1to0HiddenNodeSketchOfProof}
\end{minipage}
\end{figurebox}
When removing the columns associated with edge $6$ in $\tilde{\Skeleton}''$, we remark that the resulting matrix is associated with a DAG ReLU network whose graph is the same as $\NN$ but without the edges $5$ and $6$ and where the node $u$ has been transformed into an output node. The pathlifting and DAG associated with such a graph are provided in  \cref{eq:tildeSkeletonPrimePrime,fig:1to0HiddenNodeSketchOfProof}.
\end{myproof}
Let $F$ be the skeleton matrix of the DAG ReLU network of \cref{fig:1to0HiddenNodeSketchOfProof}, then, because we have emptied the second row and last column of $\tilde{\Skeleton}''$, we have $\rk(\tilde{\Skeleton}'') = \rk(F) + 1$, by induction $\rk(F) = \dimTheta - 2$, it follows that
\begin{align}
    \rk(\Skeleton) = \rk(\tilde{\Skeleton}'') = \rk(F) + 1 = \dimTheta -2 + 1 = \dimTheta - h
\end{align}
which concludes the sketch.

%% file: Experiment.tex
\section{Experiment}\label{sec:experiment}
In this last section of the paper, we use the Python package available at this \href{https://gitlab.inria.fr/mverbock/pathliftingandskeleton/-/tree/main?ref_type=heads}{GitLab} to compare the computational cost of computing $\partial_\theta\PathL(\theta)\in\R(\dimPathL,\dimTheta)$ when such computation is done using \cref{prop:PathLJacobian} or by backpropagation with the Python package \texttt{torch}. We provided the details on the experiment and the pseudo code for the computation of the pathlifting $\PathL(\theta)$ and the skeleton matrix $\Skeleton$ for an arbitrary feed-forward network in \cref{sec:suppexperiment}.

We consider a feed forward ReLU network $f_\theta :\R(\dimx)\mapsto\R(\dimy)$ with $\deep$ hidden layers and where each hidden layer $l$ has $\width_l$ hidden nodes. We denote $\dimTheta$ the dimension of the network parameters $\theta$ and $\dimPathL$ the dimension of the pathlifting vector $\PathL(\theta)$.  We consider two methods to compute the Jacobian $\partial_\theta\PathL(\theta)$.
\begin{itemize}
    \item[$\flower$] The method using the skeleton matrix with $\partial_\theta\PathL(\theta) = \diag(\PathL(\theta))\Skeleton\diag(\frac{1}{\theta})$ of \cref{prop:PathLJacobian}.
    \item[$\flower$] The method using PyTorch automatic differentiation \cite{AutoDiff} at each index of $\PathL(\theta)$.
\end{itemize}
The implementations of those two methods are provided in \cref{algo:JacobianSkeleton,algo:JacobianNoSkeleton}.

In the left figure \cref{fig:JacobianComputationTime}, we plot the time in seconds to compute the pathlifting Jacobian using the two described methods. 
This computational time is displayed as a function of the variable $\dimPathL\dimTheta$, which is obtained by taking $\deep\in\llbracket1, 3\rrbracket$ and $\width_l$ is constant through the hidden layers and equal to $w\in\llbracket 1, 8\rrbracket$. Each setting is evaluated 5 times, and we plot the mean as a solid line and the standard deviation as a transparent line. As expected, computing the Jacobian with the skeleton matrix is faster than using the naive method, which loops over the dimensions of the pathlifting. To fairly compare the two methods, and because \cref{algo:JacobianSkeleton} uses the skeleton matrix as an additional input, we also plot the time complexity of computing the skeleton matrix (\cref{algo:SkeletonMatrix}) at the bottom of \cref{fig:JacobianComputationTime}. We see that such computational cost grows with the number of paths; however, such cost can easily be neglected because the skeleton matrix does not depend on $\theta$ and can be computed once and for all at the network initialization and reused for the different computations of $\partial_\theta\PathL(\theta)$.
\begin{remark}
    This computational gain holds when the manipulated tensors $\PathL(\theta)$ and $\Skeleton$ can be stored in memory. However, for theoretical studies with relatively small networks, the gain in time is huge.
\end{remark}

\begin{minipage}{0.54\textwidth}
    \centering
    \includegraphics[width=0.85\textwidth]{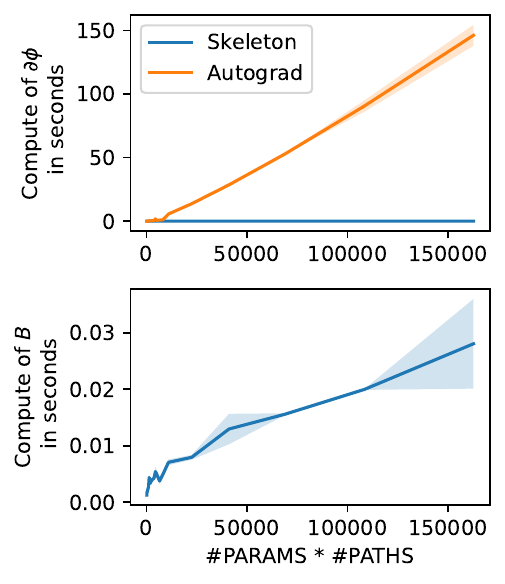}
    \captionof{figure}{
Times in seconds to compute the Jacobian matrix $\partial_\theta\PathL(\theta)$ and the skeleton matrix $\Skeleton$ using \cref{algo:JacobianSkeleton,algo:JacobianNoSkeleton} as a function of the variable $\dimPathL\dimTheta$. [~$\dimx=5, \dimy=2$]}
    \label{fig:JacobianComputationTime}
\end{minipage}
\hfill
\begin{minipage}{0.43\textwidth}
   \begin{minipage}{\textwidth}
    \noindent\rule[7pt]{\linewidth}{0.4pt}
        \normalem 
        \begin{algorithm}[H]
        \KwData{$\theta$ the current parameters.$\phi :=$ the pathlifting at point $\theta$, $\Skeleton$ the skeleton matrix}
        \KwResult{$J$ the Jacobian of $\phi$ at $\theta$}
        $u \leftarrow \frac{1}{\theta}$\;
        $J = \texttt{torch.einsum}('\;\dimPathL, \dimPathL\dimTheta, \dimTheta->\dimPathL\dimTheta\;', \phi, \Skeleton, u)$\;
        \caption{Jacobian of $\PathL(\theta)$ with  $\Skeleton$}
        \label{algo:JacobianSkeleton}
        \end{algorithm}
    \noindent\rule[7pt]{\linewidth}{0.4pt}
\end{minipage}

\begin{minipage}{\textwidth}
    \noindent\rule[7pt]{\linewidth}{0.4pt}
        \normalem 
        \begin{algorithm}[H]
        \KwData{$\theta$ the current parameters.$\phi :=$ the pathlifting at point $\theta$, }
        \KwResult{$J$ the Jacobian of $\phi$ at $\theta$}
        J = \texttt{zeros}($\dimPathL, \dimTheta$)\;
        \For{$i \in [\dimPathL]$}{
                $J[i,:] = \texttt{torch.autograd.grad}($$\phi_i, \theta,$ $retain\_graph=True)$\;
        }
        \caption{Jacobian of $\PathL(\theta)$ with autograd}
        \label{algo:JacobianNoSkeleton}
        \end{algorithm}
    \noindent\rule[7pt]{\linewidth}{0.4pt}
\end{minipage}
\end{minipage}

%% file: Conclusion.tex
\section{Conclusion}
We hope this paper has helped the reader to grasp the key element of the induction and the structure of the pathlifting Jacobian of a DAG ReLU network. 

The computation of pathlifting and its Jacobian remains an issue for large networks; however, the provided module is a useful tool for testing conjectures and manipulating the pathlifting and its Jacobian on toy networks.

\section*{Acknowledgments}
I gratefully acknowledge my supervisor Rémi Gribonval for his guidance during this work. \\
I gratefully acknowledge the support of the Centre Blaise Pascal's IT test platform at ENS de Lyon (Lyon, France) for Machine Learning facilities. The platform operates the SIDUS solution \cite{CBP} developed by Emmanuel Quemener.\\
This work was supported in part by the AllegroAssai ANR19-CHIA-0009 project of the French Agence Nationale de la Recherche (ANR) and by the SHARP ANR project ANR-23PEIA-0008 in the context of the France 2030 program.

%% file: Proofs.tex
\section{Proofs}\label{sec:proofs}
\input{Rank_phi_equals_rank_Skeleteton.tex}

\input{Rank_Skeleton}

\input{DetailsOnExperiement.tex}

%% file: Rank_phi_equals_rank_Skeleteton.tex
\begin{proposition}\label{proof:PhiLinearizesf}
    \PhiLinearizesf
\end{proposition}
\begin{proof}
    Let $\{x_i\}_{i=1}^n\in\R(\dimx, n)$ be a finite dataset and let $\mu$ be a Lebesgue measure on the network parameters. 
    We restrict our proof to the networks whose output nodes are associated with the identity activation function. This is not restrictive, as any DAG ReLU network can be mapped to one of these.

    Without loss of generality, we suppose that the input nodes associated with the biases have been concatenated to the input, i.e $x_i \leftarrow \begin{pmatrix} x_i \\ 1\end{pmatrix}\in\R(\dimx)$.\\
    To construct the appropriate linear function $\mathscr{L}$, we first construct a neighborhood $O$ of $\theta$ on which the activation pattern of the network is constant, and then construct the appropriate linear function.
    
    \begin{myproof}
    \textbf{Construction of the neighborhood.} 
    Let $\theta\in\R(\dimTheta)$ be a parameter initialization. 
    Let $A(\theta, x)\in\R(s)$ be the pre-activation of the network at point $x$ with parameter $\theta$, i.e the concatenation of all the pre-activations of all the nodes plus the input. The function $A(\cdot, \cdot)$ is continuous and for all $x$ function $A(\cdot, x)$ is a finite selection of polynomials.\\

    For $i\in\{1, \cdots, n\}$, for $j$ an index of $A(\theta,x_i)$. We have the following :
    \begin{itemize}
        \item If $A(\theta, x_i)_j \ne 0$, then by continuity of $A(\cdot, x_i)$, there exists a neighborhood of $\theta$ denote by $O_{i, j}$ such that for any $\theta'\in O_{i, j}$, $A(\theta', x_i)_j \ne 0$ and thus the sign of $A(\theta', x_i)_j$ is constant on $O_{i, j}$.
        \item If $A(\theta, x_i)_j = 0$, then we define $Z = \{\theta'\in\R(\dimTheta) \;|\; A(\theta', x_i)_j = 0\}$. 
                \begin{itemize}
                    \item If $\theta \in Z^\circ$ the interior of $Z$, then by continuity of $A(\cdot, x_i)$, there exists a neighborhood of $\theta$ note $O_{i, j}$ such that for any $\theta'\in O_{i, j}$, $A(\theta', x_i)_j = 0$ and thus the sign of $A(\theta', x_i)_j$ is constant on $O_{i, j}$.
                    \item If $\theta \in \partial Z$ the boundary of $Z$, then we shall prove that this set is of null Lebesgue measure. 
                \end{itemize}
    \end{itemize}

    \underline{The set $\partial Z$ is of null Lebesgue measure} : If $\partial Z$ is empty then it is of null Lebesgue measure. Otherwise, we proceed as follows. Because the function $\theta' \mapsto A(\theta', x_i)_j$ is piecewise polynomial and associated with a finite number of polynomial functions, let $p_1, \cdots, p_m$ be the polynomial functions such that for each $\theta'\in\R(\dimTheta)$, there exists $k\in\{1, \cdots, m\}$ such that $A(\theta', x_i)_j = p_k(\theta')$. 

    If $\theta\in\partial Z$, then because $Z$ is closed and by continuity of $A$, we have 
    \begin{align}
        A(\theta, x_i)_j = 0\;.
    \end{align}
    Furthermore, it exists $\{\theta_p\}_{P> 0}$ such that $\theta_p\rightarrow \theta$ and $A(\theta_p, x_i)_j \ne 0$ for all $p$. Because the polynomials are distinct, it exists $\{\theta_{l_p}\}$ a sub sequence and $k\in\{1, \cdots, m\}$ such that for $l_p$ large enough, $A(\theta_{l_p}, x_i)_j = p_k(\theta_p) \ne 0$ and in particular $p_k$ is a nonzero polynomial function and $\theta$ is one of its roots by continuity of $p_k$:
    \begin{align}
        &p_k \ne 0 &&p_k(\theta) = 0\;.
    \end{align}
    Because the set of roots of a non null polynomial function is of null Lebesgue measure, we have that $\theta$ belongs to a set of null Lebesgue measure. As a consequence, the set $\partial Z$ is of null Lebesgue measure.\\
    \\
    By construction, the set $O = \bigcap_{i=1}^n\bigcap_{j=1}^s O_{i, j}$ is neighborhood of $\theta$, and almost surely on $\theta$, the sign of the pre-activation pattern of the network is constant on $O$. Furthermore, we define the matrix $F $ the element-wise application of the indictor of $\R^+$ at index $j$ when the associated activation function is ReLU and $1$ otherwise, i.e
    $F_{i, j} =\mathbf{1}_{A(\theta, x_i)_j > 0}$ if the node $j$ is associated with a ReLU activation function and $1$ otherwise.
    \end{myproof}
    \begin{myproof}
        \textbf{Construction of the linear function $\mathscr{L}$.} Each path of the network connects one input node to one output node and is active at point $x_i$ if and only if all the indices of the $i^{th}$ columns of $F$, i.e. $act(A(\theta, x_i))$, associated to that path are equal to one. Those indices are also the ones of the nodes crossed by the path.\\

        For $\pathn$ a path, let $j^p_x$ be the input node index of $\pathn$ and $j^p_y$ be the output node index of $\pathn$. Let $J_p$ be the set of indices of the hidden nodes crossed by the path $\pathn$, each of them associated with an index in $A(\theta, x_i)_{j_y}$. We define 
        \begin{align}
            c^i_p = \prod_{j\in J_p}F_{i, j}
        \end{align}
        the product of all the activity patterns of the nodes crossed by path $\pathn$ for point $x_i$. The coefficient $c_p^i$ is constant for  $\theta' \in O$ and indicates if the path is active at point $x_i$.\\

        For $\theta'\in O$, let $\PathL(\theta')\in\R(\dimPathL)$ be the vector of paths at parameter $\theta'$, and let $j\in\{1, \cdots, \dimy\}$ be an output index. It follows that
        \begin{align}
            f_{\theta'}(x_i)_{j} = \sum_{p=1}^\dimPathL\PathL(\theta')_p \; c_p^i\mathbf{1}_{j=j_y^p}\;x_i[j^p_x]
        \end{align}
        where $\mathbf{1}_{j=j_y^p}$ is the indicator function that is equal to $1$ if $j=j_y^p$ and $0$ otherwise.

        Let $C_i^j := \begin{pmatrix}
            c^i_1\mathbf{1}_{j=j_y^1}\;x_i[j^1_x]\\
            \vdots\\
            c^i_{\dimPathL}\mathbf{1}_{j=j_y^{\dimPathL}}\;x_i[j^{\dimPathL}_x]
        \end{pmatrix}\in\R(\dimPathL)$ and let $\mathscr{L}$ be the linear function defined by
        \begin{align}
            \forall u\in\R(\dimPathL),\quad \mathscr{L}(u) = \bigg\{\begin{pmatrix}
            \scalarp{C_i^1, u}\\
            \vdots\\
            \scalarp{C_i^\dimy, u}
            \end{pmatrix}\bigg\}_{i=1, \cdots, n}\;.
        \end{align}
    It follows that for any $\theta'\in O$, we have
    \begin{align}
        \{f_{\theta'}(x_i)\}_{i=1}^n =  \mathscr{L}\big(\PathL(\theta')\big) \;.
    \end{align}
    \end{myproof}
\end{proof}
\begin{lemma}\label{proof:phiequalsB}
    \lemphiequalsB
\end{lemma}
\begin{proof}
Let $\NN$ be a DAG ReLU network, and let $\Skeleton\in\R(\dimPathL, \dimTheta)$ its skeleton matrix and $\PathL(\theta)$ its vector of paths. 

Let $\theta\in\R(\dimTheta)$ be a parameter initialization and without loss of generality, we suppose that $\theta$ has no null parameters (this event being of null Lebesgue measure). 

Let $S = \diag(\sign(\PathL(\theta)))\in\R(\dimPathL, \dimPathL)$ the diagonal matrix with coefficients $S_{ii} = \sign(\PathL(\theta)_i)$.

\begin{myproof}
To prove $\cref{proof:phiequalsB}$, we shall prove the equality at each index of $\PathL(\theta)$. Let $j\in \{1, ..., \dimPathL\}$, we have
\begin{align*}
 \bigg[S\exp\big(\Skeleton\log(|\theta|)\big)\bigg]_j &= \sum_{i=1}^\dimPathL S_{ji}\bigg[\exp\big(\Skeleton\log(|\theta|)\big)\bigg]_{i, 1}\\
    &= S_{jj}\bigg[\exp\big(\Skeleton\log(|\theta|)\big)\bigg]_{j, 1}\\
    &= \sign\big(\PathL(\theta)_j\big)\exp\bigg(\sum_{k=1}^{\dimTheta}\Skeleton_{jk}\log(|\theta_{k}|)\bigg)\;.
\end{align*}
The sum on $k$ ranges through all the parameters of the network and the coefficient $\Skeleton_{jk}$ is equal to $1$ iff the parameter $k$ is in the expression of the path $j$. It follows that 
\begin{align}
    \sum_{k=1}^{\dimTheta}\Skeleton_{jk}\log\left(|\theta_k|\right) = \log\big(\left|\PathL(\theta)_j\right|\big)\label{eq:expBlog=phi}
\end{align}
and 
\begin{align*}
 \bigg[S\exp(\Skeleton\log(|\theta|))\bigg]_j &= \sign\big(\PathL(\theta)_j\big)\left|\PathL(\theta)_j\right| = \PathL(\theta)_j\;.
\end{align*}
\end{myproof}
\end{proof}
\begin{proposition}\label{proof:PathLJacobian}
    \propPathLJacobian
\end{proposition}
\begin{proof}
We prove the proposition by performing a first-order Taylor development of $\PathL$ on a neighborhood $\mathcal{V}$ of $\theta$.

Let $\theta\in\R(\dimTheta)$ be a parameter that has no null coordinates, let $\epsilon = \argmin_j{|\theta_j|}$ and let $\mathcal{V} = \mathcal{B}_\theta(\epsilon/2)$ be the ball centered at $\theta$ of radius $\epsilon/2$ be a neighborhood of $\theta$. By construction for all $\theta'\in\mathcal{V}$, $\theta'$ has no null coordinates.
\begin{myproof}
\textbf{First order Taylor development of $\PathL$ on $\mathcal{V}$.} We use the row of the skeleton matrix $\Skeleton = \left(\Skeleton_1, ..., \Skeleton_{\dimPathL}\right)^T$ where $\Skeleton_j\in\R(\dimTheta)$.
Let $h \in\R(\dimTheta)$ such that $\theta + h\in\mathcal{V}$,  starting from the expression of $\PathL$ provided by \cref{proof:phiequalsB}, we have
\begin{align}
 \PathL(\theta + h)_j &= \bigg[S\exp\big(\Skeleton\log(|\theta + h|)\big)\bigg]_j\\
     &= S_{j,j}\exp\big(\Skeleton_j\log(|\theta + h|)\big)\\
    &=S_{j,j}\exp\bigg(\Skeleton_j\log(|\theta|) + \Skeleton_j\log( |\mathbf{1}_\dimTheta + \diag\big(\frac{1}{\theta}\big)h|)\bigg)\label{eq:TaylorExpansionabs} 
\end{align}
where $\mathbf{1}_\dimTheta\in\R(\dimTheta)$ is the vector full of ones and $\diag(\frac{1}{\theta})h =\begin{pmatrix}
\frac{h_1}{\theta_1}\\ .\\ \frac{h_\dimTheta}{\theta_\dimTheta}
\end{pmatrix}$.
We remark that, because $\theta +h\in\mathcal{V} = \mathcal{B}_\theta(\epsilon/2)$ we have for $k =1, ..., \dimTheta$,  $0 <|h_k| < |\theta_k| < 1 $ and thus
\begin{align}
   0 <  1 + \frac{h_k}{\theta_k} \;.
\end{align}
It follows that $|\mathbf{1}_\dimTheta+\diag(\frac{1}{\theta})h| = \mathbf{1}_\dimTheta + \diag(\frac{1}{\theta})h$; thus, we can remove the absolute value in the exponential of \cref{eq:TaylorExpansionabs}. We continue the Taylor development from \cref{eq:TaylorExpansionabs}.
\begin{align*}
 \PathL(\theta + h)_j&=S_{j, j}\exp\left(\Skeleton_j\log(|\theta|) + \Skeleton_j\diag(\frac{1}{\theta})h + o(\norm{h})\right)\\
    &=S_{j, j}\exp\left(\Skeleton_j\log(|\theta|)\right) \exp\left(\Skeleton_j\diag(\frac{1}{\theta})h + o(\norm{h})\right)\\
    &=S_{j, j}\exp\bigg(\Skeleton_j\log(|\theta|) \bigg)\left( 1 +\Skeleton_j\diag(\frac{1}{\theta})h + o(\norm{h})\right) \;.
\end{align*}
Using again \cref{proof:phiequalsB}, i.e  $S_{j, j}\exp\left(\Skeleton_j\log(|\theta|)\right) = \PathL(\theta)_j$, it follows that
\begin{align*}
 \PathL(\theta + h)_j &=  \PathL(\theta)_j + \PathL(\theta)_j\Skeleton_j\diag(\frac{1}{\theta})h + o(\norm{h})\;.
\end{align*}
By identification of the linear term in $h$, it follows that the $j^{th}$ row of $\partial_\theta \PathL$ is equal to $\big[\partial_\theta \PathL\big]_j = \PathL(\theta)_j\Skeleton_j\diag(\frac{1}{\theta})\in\R(1, \dimTheta)$ and thus
\begin{align*}
    \partial_\theta\PathL&=\diag(\PathL(\theta))\Skeleton\diag(\frac{1}{\theta}) \; .
\end{align*}
\end{myproof}

\end{proof}

\begin{corollary}\label{proof:rkPathLSkeleton}
    \corrkPathLSkeleton
\end{corollary}
\begin{proof}
For $\mu$ a measure on $\theta\in\R(\dimTheta$), using \cref{proof:PathLJacobian} we have $\mu$ almost surely :
\begin{align*}
    \partial\PathL = \diag(\PathL)\Skeleton\diag(\frac{1}{\theta})
\end{align*}
As all the coordinates of $\theta$ are different from $0$ $\mu-$ almost surely, thus $\mu-$ almost surely the matrix $\diag(\frac{1}{\theta})$ and $\diag\left(\PathL\right)$ are invertible  and we have $\rk(\partial_\theta\PathL) = \rk(\Skeleton)$.\\
\end{proof}

%% file: Rank_Skeleton.tex
\begin{theorem}\label{proof:rkSkeleton}
    \thrkSkeleton
\end{theorem}
We prove \cref{th:rkSkeleton} by induction on the number of hidden neurons  in the graph $\Graph$ of the network $\NN$ which we denote by $h = \#H$.
\begin{proof}
Let $\NN$ be a neural network with $h$ hidden nodes and let $\Graph$ be its DAG. Let $N = \{n_1, \cdots, n_m \}$ be its set of nodes, $E=\{e_1, \cdots, e_\dimTheta\}$ its set of edges, and $\Skeleton$ its skeleton matrix.
\begin{myproof}
\textbf{$\mathbf{h=0}$} : 
An example of such a graph is shown in \cref{fig:NoHiddenNode_rkPhi}. The nodes are indexed with integers and the edges with letters.\\
\begin{minipage}{0.35\textwidth}
    To compute the rank of the skeleton matrix $\Skeleton$ of this network, we apply invertible row operations on it, then evaluate the rank of the resulting matrix. Indeed, performing row operations on a matrix is equivalent to applying invertible matrices on its left side, and the resulting matrix after such operations is of the same rank as the original matrix.
    We choose the row operations that take advantage of some specific properties of the skeleton matrix $\Skeleton$ and which are embodied by the following lemma.
\end{minipage}
\hfill
\begin{minipage}{0.6\textwidth}
    \begin{figurebox}
    \centering
    \includegraphics[width = 0.6\textwidth]{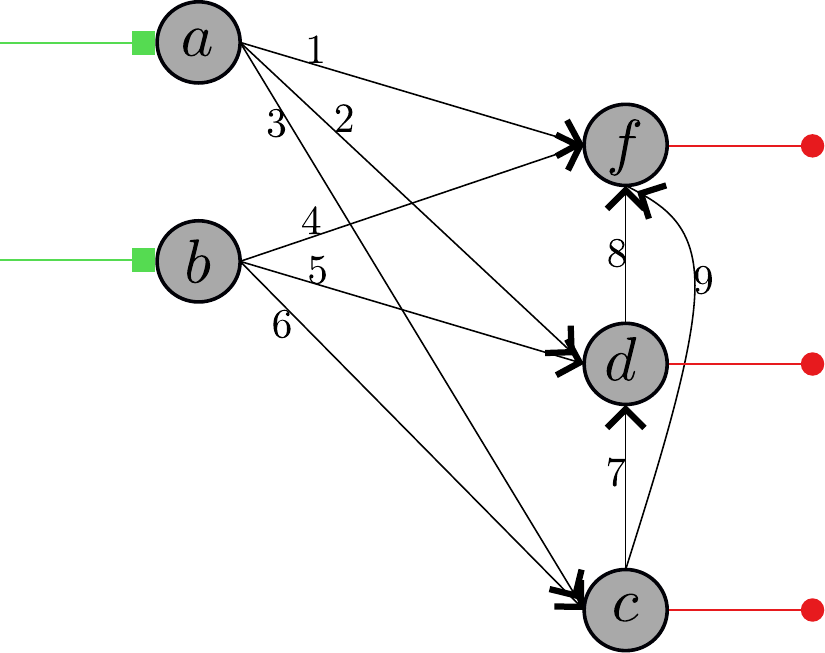}
    \captionof{figure}{Example of a DAG of a network without a hidden node: the green and the red arrows represent the input nodes and the output nodes, respectively. The nodes are indexed with letters $\Nodes = \{n_a, n_b, n_c, n_d, n_f\}$ and the edges with integers $\Edges = \{e_1, e_2, e_3, e_4, e_5, e_6, e_7, e_8, e_9\}$.}
    \label{fig:NoHiddenNode_rkPhi}
    \end{figurebox}
\end{minipage}\\
\vspace{0.5cm}\\
    \begin{lemma}\label{proof:lemH0}
        \lemBminusBHzero
    \end{lemma}
    \begin{myproof}    
    \begin{proof}
    
Let $l > 1$ and $\pathn$ of length $l$ going through the ordered list of edges $\edge_1, \ldots, \edge_l$. Because the graph is acyclic and has no hidden node, the edge $\edge_1$ connects an input node to an output node, and for all $j = 2, \ldots, l$ the edge $\edge_j$ connects two output nodes. As a consequence, the edge $\edge_{l-1}$ is directed to an output node and thus the path going through the edges $\edge_1, ..., \edge_{l-1}$ ending at $e_{l-1}$ exists.
    \end{proof}
    \end{myproof}
    
    Using this symmetry, we construct the algorithm \cref{algo:rowoperationSkeleton} which takes as input the matrix $\Skeleton$ and outputs a matrix $\tilde{\Skeleton}$ that is of the same rank as $\Skeleton$. 

\noindent\rule[7pt]{\linewidth}{0.4pt}
    \normalem 
    \begin{algorithm}[H]
     \KwData{$\Skeleton$ the skeleton matrix, $l_{max}$ the maximum path length in $\Graph$}
     \KwResult{matrix $\tilde{\Skeleton}$ of the same rank as $\Skeleton$}
     \For{$l = l_{max}, l_{max} - 1, \ldots, 2$}{
      \For{$\pathn$ of length $l$}{
      let $\tilde{\pathn}$ be the path defined in \cref{proof:lemH0} for path $\pathn$\;
      $\Skeleton[\pathn, :] \xleftarrow{} \Skeleton[\pathn, :] - \Skeleton[\tilde{\pathn}, :]$\;
      }
     }
     \caption{Row operations on matrix $\Skeleton$}
     \label{algo:rowoperationSkeleton}
    \end{algorithm}
\noindent\rule[7pt]{\linewidth}{0.4pt}
    
    The matrix $\tilde{\Skeleton}$ outputted by \cref{algo:rowoperationSkeleton} is of the same rank as $\Skeleton$ and each of its rows has a unique non-null element corresponding to the last edge crossed by the paths associated with that row (\cref{proof:lemH0}). As a consequence the rank of $\tilde{\Skeleton}$ is equal to number of its non-null columns.\\
    
    As all the edges of the networks are directed toward an output, every edge is a last edge of one path, thus there is no zero columns in $\tilde{\Skeleton}$ thus,
   the rank of $\tilde{\Skeleton}$ equals the number of columns, that is the number of parameters in the network. It follows that
    \begin{align*}
        \rk(\Skeleton) = \rk(\tilde{\Skeleton}) = \dimTheta - 0\;.
    \end{align*}
\end{myproof}
\begin{myproof}
    \textbf{Induction :} We suppose \cref{proof:rkSkeleton} is true for any DAG ReLU network with at most $h-1$ hidden nodes.
    Let $\NN$ be a DAG ReLU network with $h$ hidden nodes. We note $\{\node_1,\ldots, \node_m\}$ and $\{\edge_1, \dots, \edge_\dimTheta\}$ the lists of its nodes and edges.\\
\begin{minipage}{0.5\textwidth}
    Because there is no cycle in the graph of $\NN$, we index its nodes with a topological sort of its graph with the constraint that the input nodes are indexed first.
    
    We choose $u$ as the larger hidden node index in the topological sort of the graph and let $k>u$ be the smaller integer such that node $n_{k}$ is connected to node $n_u$. 
    Let $\edge_{\leftrightarrow}$ be the edge connecting node $\node_u$ and node $\node_{k}$. \uline{The key element of the proof is to remove the edge $e_{\leftrightarrow}$ from the graph of $\NN$ and transform $n_u$ into an output node. Those actions would remove one hidden node in the graph of $\NN$ and would allow the use of the induction hypothesis}. The technique of the proof mainly relies on the fact that the column on the skeleton matrix at edge $\edge_{\leftrightarrow}$ is linearly dependent on the columns associated with the incoming and outgoing edges of node $u$.  
    
    The steps of the proof are similar to the sketch in \cref{sec:sketch}. First, we re-order the columns and rows of the skeleton matrix, then we remove the edge $e_{\leftrightarrow}$, and finally we transform $n_u$ into an output node. In the sketch, those two last steps are done simultaneously.
\end{minipage}
\hfill
\begin{minipage}{0.45\textwidth}
\begin{figurebox}
    \includegraphics[width=0.9\textwidth]{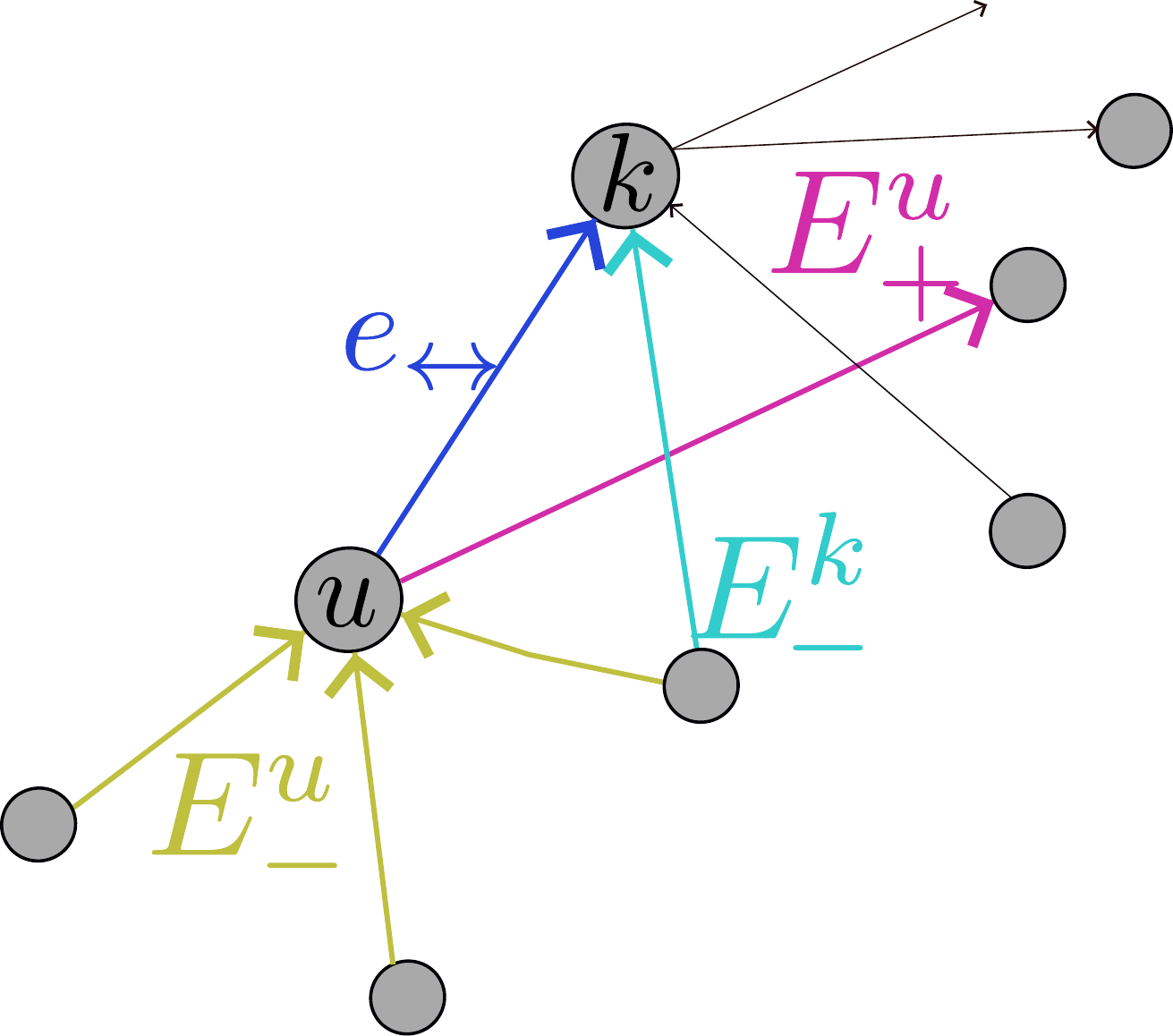}
    \captionof{figure}{DAG of a network with $h$ hidden nodes. The set of edges $\color{colorEum}{E_{-}^u}$, $\textcolor{colorEup}{E_{+}^u}$ and $\textcolor{colorEkm}{E_{-}^k}$ are colored in dark yellow, purple and cyan. The blue edge $\textcolor{colore}{e_\leftrightarrow}$ is the edge connecting the hidden node $n_u$ to the output node $n_k$.}
    \label{fig:HiddenNeuron}
\end{figurebox}
\end{minipage}
\vspace{0.5cm}\\
\textbf{1. Removing the edge $\edge^{\leftrightarrow}$}\\
Note $E_-^{u}$ the set of the in-going edges of $n_u$ and $ E^ {u} _ {+} $ the set of the outgoing edges of $n_u$ but without $\edge^{\leftrightarrow}$ and $E_{-}^k$ the set of in-going edges of the nodes $\node_{k}$ but without $\edge^{\leftrightarrow}$.
Those different sets are indicated with color in \cref{fig:HiddenNeuron}. 
We now present the symmetries of $\Skeleton$ with the following lemmas :
\begin{lemma}\label{lem:UniqueCompositionP}
 If path $\pathn$ goes through node $\node_u$ then $\pathn$ passes by a unique edge of $E_{-}^u$ and a unique edge of $\{\edge^{\leftrightarrow}\}\cup E_+^u$.
\end{lemma}
\begin{myproof}    
    \begin{proof}
 By definition, $E_{-}^u$ is the set of all the incoming edges of $n_u$, and $\{\edge^{\leftrightarrow}\}\cup E_+^u$ is the set of all the outgoing edges of node $n_u$, which concludes the proof.
    \end{proof}
\end{myproof}
\begin{lemma}\label{lem:inEkNotinEh}
 Let $\pathn$ be a path of $\NN$, then $\pathn$ cannot simultaneously go through an edge of $E_{-}^u\cup E_+^u\cup\{\edge^{\leftrightarrow}\}$ and an edge of $E_{-}^k$. 
\end{lemma}
\begin{myproof}
    \begin{proof}
        \textbf{We argue by contradiction}.
    
        We suppose that it exists a path $\pathn$, such that $\pathn$ goes through an edge of $E_{-}^u\cup E_{+}^u\cup \{\edge^{\leftrightarrow}\}$ and through an edge of $E_{-}^k$.
                
        Let $\pathn$ be such path. By construction $\pathn$ goes through node $u$ and node $k$. Because $\pathn$ goes through an edge in $E_{-}^k$ and because $\Graph$ has no cycle then path $\pathn$ does not cross the edge $\edge^{\leftrightarrow}$. As a consequence, the path $\pathn$ crosses one intermediate node between $\node_u$ and $\node_k$. Let $\node_u, \node_\alpha, ..., n_{k}$ be the ordered list of nodes associated to sub path of $\pathn$ starting at $\node_u$ and ending at $\node_{k}$. Because we have indexed the node using the topological sort, we have $u < \alpha < k$. On the other hand, $k$ is the smaller integer for which $\node_{k}$ is connected to node $\node_u$, so $k \leq \alpha$, absurd. 
    \end{proof}
\end{myproof}

\begin{minipage}{0.4\textwidth}
    We now re-order the columns of $\Skeleton$ such that it first columns correspond to $e^{\leftrightarrow}$, $E_{+}^u$, $E_{-}^u$ and $E^k_{-}$. We also reorder its lines so that the first lines correspond to all the paths crossing both nodes $\node_{u}$ and $\node_{k}$, then all the paths crossing $\node_{u}$ but not by node $n_{k}$. The skeleton matrix is of the form of the right figure, where the blue zeros are induced by \cref{lem:UniqueCompositionP}, the magenta zeros are induced by \cref{lem:inEkNotinEh}, and the green zeros are induced by \cref{lem:inEkNotinEh}. The colored dots correspond to a repetition of 0 in the matrix. 
\end{minipage}
\hfill
\begin{minipage}{0.55\textwidth}
    \begin{align*}
        \Skeleton &= 
        \begin{matrix}
            \begin{matrix}
                \;\hspace{-0.2cm}\edge_{\leftrightarrow} & & &\!E_+^u& &\hspace{0.4cm}E_-^u &\hspace{0.2cm}E_{-}^k&.&
            \end{matrix}\\
            \OrderSkeleton
        \end{matrix}
    \end{align*}
\end{minipage}
\vspace{0.5cm}\\
Furthermore, as a path crosses at most one edge of $E_{-}^u$ ( \cref{lem:UniqueCompositionP}), the matrices indexed on letter $A$ are of the same form as the first block matrix of $\Skeleton$ with only one non-null element by row as
\begin{align}
    \DiagZOnes \label{eq:AdiagOne}
\end{align}
and the sum of its columns gives the vector full of ones.\\
 
We now perform two types of column operations in the skeleton matrix. First, we sum all the columns associated with the edge $E_{-}^u$ to obtain a vector full of ones on top and full of zeros at its bottom, and subtract it from the first columns of $\Skeleton$. We note $\Skeleton'$ the results of such an operation. Then, we add the columns associated with the set $E_{+}^u$ to the first columns of $\Skeleton'$ and note $\tilde{\Skeleton}$ the result of this operation. The corresponding matrix are 
\begin{align*}
   \hspace{-0.4cm}\Skeleton' = \IntermediatSkeleton,\;\tilde{\Skeleton} = \ModifiedSkeleton
\end{align*}\\
Where the modifications are indicated in red.
We see in the matrix $\tilde{B}$ that the connection $\edge^{\leftrightarrow}$ has been removed in the sense that the column which was previously associated with it is full of zeros.
\vspace{0.5cm}\\
\textbf{2. Transform $n_u$ into an output node}\\
We now transform $\tilde{\Skeleton}$ into a skeleton matrix of a network with $h-1$ hidden neurons. To do so, we should remove in $\tilde{\Skeleton}$ the ones corresponding to the end of the paths crossing the edge $\edge^{\leftrightarrow}$ and not ending at node $n_k$. This part of the proof is similar to the induction's initialization. 

Let $\Paths$ the set of path crossing $\edge^{\leftrightarrow}$ and not ending at $n_k$. We take advantage of the choice of node $u$ in the topological sort with the following lemma.
\begin{lemma}\label{proof:ppcoincide}
    \lemppcoincideSK
\end{lemma}
\begin{myproof}
    \begin{proof}
        The proof is the same as the proof of \cref{proof:lemH0} the only difference is that we should consider matrix $\tilde{\Skeleton}$. Considering the writing of $\tilde{\Skeleton}$, we see that the only difference is that the column associated with the edge $\edge^{\leftrightarrow}$ is full of zeros, thus one can replace $\Skeleton$ by $\tilde{\Skeleton}$ and the result still holds.
       \end{proof}
\end{myproof}
We now perform row operations on $\tilde{\Skeleton}$ with \cref{algo:rowoperationSkeletonReccurence}. \\

\noindent\rule[7pt]{\linewidth}{0.4pt}
    \normalem 
    \begin{algorithm}[H]
     \KwData{$\tilde{\Skeleton}$ the modified skeleton matrix, $\Paths$ the list of path crossing edge $\edge_{\leftrightarrow}$ and not ending at node $n_k$, $l_{min}, l_{max}$ the minimum and maximum length of a path in $\Paths$.}
     \KwResult{matrix $\tilde{\Skeleton}'$ of the same rank as $\tilde{\Skeleton}$}
     \For{$l\in\{ l_{max}, l_{max} -1,  \cdots, l_{min}\}$}{
     \For{$p\in \Paths$ and $p$ of length $l$}{
      Let $\tilde{\pathn}$ be the  path of \cref{proof:ppcoincide} coinciding with $\pathn$ on its $l-1$ first nodes.
      $\tilde{\Skeleton}[\pathn, :]\leftarrow \tilde{\Skeleton}[\pathn, :] - \tilde{\Skeleton}[\tilde{\pathn}, :] $
      }
      }
     \caption{Row operations on matrix $\tilde{\Skeleton}$}
     \label{algo:rowoperationSkeletonReccurence}
    \end{algorithm}
\noindent\rule[7pt]{\linewidth}{0.4pt}
The matrix $\tilde{\Skeleton}'$ outputted by \cref{algo:rowoperationSkeletonReccurence} is of the same rank as $\tilde{\Skeleton}$. Let $E^{end}$ be the set of edges corresponding to the last edge of a path in $\Paths$. For each edge $e$ in $E^{end}$, there exists a row in $\tilde{\Skeleton}'$ such that the only non-null element of the row is at edge $e$ (\cref{proof:ppcoincide,algo:rowoperationSkeletonReccurence}). Using this property, we remove in $\tilde{\Skeleton}'$ the dependency of all edge $e$ in $E^{end}$ for the paths not in $\mathcal{P}$. We perform such operations and obtain the matrix $F$. Ordering its columns starting with the set of edges $E^{end}$, and $\edge_{\leftrightarrow}$, this matrix is of the form 
\begin{align*}
    F = \begin{matrix}
        \begin{matrix}
            E^{end}&\edge_{\leftrightarrow}&
        \end{matrix}
        \\
        \begin{pmatrix}
       & A & &\mathbf{0} &\mathbf{0} \\
       & \mathbf{0} & &\mathbf{0} &H\\
       & \mathbf{0} & &\mathbf{0} & D
        \end{pmatrix}
    \end{matrix}
\end{align*}
where the rows $\begin{pmatrix} A & \mathbf{0} &\mathbf{0} \end{pmatrix}$ correspond to the rows with only one non-null element by row each of them corresponding to an edge in $E^{end}$; where the rows $\begin{pmatrix} \mathbf{0} & \mathbf{0} & H \end{pmatrix}$ corresponds to the paths ending at node $n_u$, and the rows $\begin{pmatrix} \mathbf{0} & \mathbf{0} & D\end{pmatrix}$ correspond to all the path that are not ending at node $n_u$ or not crossing node $n_u$ and whose dependency on the edges in $E^{end}$ has been removed using the rows of $A$.
It follows that
\begin{align*}
    \rk(\tilde{\Skeleton}') = \rk(F) = \rk(A) + \rk(\begin{pmatrix}H\\D\end{pmatrix}) \;.
\end{align*}
Remark that $\rk(A) = \#E^{end}$ and that the matrix $\begin{pmatrix} H \\D \end{pmatrix}$ coincides with the skeleton matrix of a network with $h-1$ hidden nodes and  with $\dimTheta - 1 -\#E^{end}$ parameters up to null and repeated rows, which do not affect the rank.
Using the induction on the matrix $\begin{pmatrix} H \\D \end{pmatrix}$ it follows that
\begin{align*}
    \rk(\Skeleton)=\rk(\tilde{\Skeleton}') &= \#E^{end} + \dimTheta - \#E^{end} - 1 - (h-1) \\
    &= \dimTheta - h\; .
\end{align*}
Remark that the reduced graph associated to the matrix $\begin{pmatrix} H \\D \end{pmatrix}$ may contain nodes that lose all their incoming edges (when the edges in $E^{end}$ are deleted) and would then be re-classified as input nodes by Definition of $I$, however this does not affect the rank of matrix $F$.
\end{myproof}

\end{proof}

%% file: DetailsOnExperiement.tex
\section{Experiment and python module}\label{sec:suppexperiment}
\subsection{Experiment settings}
The experiment has been run on the Intel Xeon Gold 6226R CPU. Calculations use \texttt{float32} precision, and the main Python library is \texttt{Pytorch} version 2.6.0. 

\subsection{Python module}
The Python code is available at this \href{https://gitlab.inria.fr/mverbock/pathliftingandskeleton/-/tree/main?ref_type=heads}{GitLab}. The two main implementations of the module are the construction of the pathlifting and the skeleton matrix for arbitrary feed-forward networks. Both functions are implemented in the Python file \texttt{PythonFiles/PathLiftingAndSkeleton.py}.
\begin{itemize}
    \item \textbf{The pathlifting} is constructed as a tensor object of dimension $(\dimy, \dimPathL_1)$, where $\dimPathL_1$ is the number of paths going to one output node. Its construction is done following the pseudo code \ref{algo:PathLifting}. 
    \item \textbf{The skeleton} is constructed as a tensor object of dimension $(\dimy, \dimPathL_1, \dimTheta)$ where $\dimTheta$ is the number of parameters in the network. Its construction is done following the pseudo code \ref{algo:SkeletonMatrix}. 
\end{itemize}
With this tensor form of $\PathL$ and $\Skeleton$, the Jacobian $\partial_\theta\PathL(\theta)$ can be computed the modified version of \cref{algo:JacobianSkeleton} that is as follows.

\begin{minipage}{\textwidth}
    \noindent\rule[7pt]{\linewidth}{0.4pt}
        \normalem 
        \begin{algorithm}[H]
        \KwData{$\theta$ the current parameters.$\phi :=$ the pathlifting at point $\theta$, $\Skeleton$ the skeleton matrix}
        \KwResult{$J$ the Jacobian of $\phi$ at $\theta$}
        $u \leftarrow \frac{1}{\theta}$\;
        $J = \texttt{torch.einsum}('\;q\dimPathL_1, q\dimPathL_1\dimTheta, \dimTheta->q\dimPathL_1\dimTheta\;\;', \phi, \Skeleton, u)$\;
        \caption{Jacobian of $\PathL(\theta)$ with  $\Skeleton$ as tensor}
        \label{algo:JacobianSkeleton}
        \end{algorithm}
    \noindent\rule[7pt]{\linewidth}{0.4pt}
\end{minipage}

Both functions use the auxiliary function \texttt{ListOfPathsAsNodesAtLayer} described in \cref{algo:ListOfPathAsNodesAtLayer}, in which the Cartesian product is computed using the itertools Python library. 

We also remark that both implementations can be enhanced using recursive approaches.
The dimensions are summarized in \cref{tab:objects_dimensions_suppExpe}.
\begin{table}[h!]
    \begin{center}
        \caption{Dimensions}
        \begin{tabular}{|c|l|}
            \hline
            Dimension   & Description \\
            \hline
            $\dimy$&number of output nodes \\
            $L$& depth of the network\\
            $\dimTheta $& dimension of the network parameter\\
            $\dimPathL_1$& number of paths ending at one output node\\
            $(\dimy, \dimPathL_1)$& dimension of the pathlifting as a tensor\\
            $(\dimy, \dimPathL_1, \dimTheta)$& dimension of the skeleton as a tensor\\
            \hline
        \end{tabular}
        \label{tab:objects_dimensions_suppExpe}
    \end{center}
\end{table}

\begin{algorithm}[H]
\DontPrintSemicolon
\caption{ListOfPathAsNodesAtLayer}

 \noindent\rule[7pt]{\linewidth}{0.4pt}
\label{algo:ListOfPathAsNodesAtLayer}

\KwData{$LayersList$: list of ordered network layers, $l_0$ a layer index, $IsBias$:Boolean indicating if the starting node is a bias node}
\KwResult{All possible paths (represented as lists of nodes) starting from layer $l_0$}

$SelectedLayers \gets \emptyset$\;
\If{$l_0 = 1$ and $not(IsBias)$}{
    \tcp{For a bias node, we do not consider the starting node of the path}
    $SelectedLayers \gets \big[\{\,n \mid n \in LayersList[1].in\_features\,\}$\big]\;
}

\For{$k \gets l_0$ \KwTo $\deep$}{
    $SelectedLayers.append\big(
    \{\,n \mid n \in LayersList[k].out\_features\}\big)$\;
}

$PathsAsNodes \gets \textsc{CartesianProduct}(SelectedLayers)$\;

\Return{$PathsAsNodes$}
\caption{ListOfPathsAsNodesAtLayer}
\label{algo:ListOfPathsAsNodesAtLayer}
\end{algorithm}
\noindent\rule[7pt]{\linewidth}{0.4pt}

\noindent\rule[7pt]{\linewidth}{0.4pt}
\begin{algorithm}[H]
\DontPrintSemicolon
\caption{PathLifting}
\label{algo:PathLifting}

\KwData{$LayersList$: ordered list of network layers, $\theta$: network parameters}
\KwResult{Pathlifting tensor $\PathL(\theta)$}

$\PathL \gets \mathbf{1}_{\dimy \times \dimPathL_1}$\;
$StoredPaths = \{\}$\;

\ForEach{output node $q$}{
    $StoredPaths[q] \gets 0$\;
}

$PathsAsNodes \gets \textsc{ListOfPathsAsNodesAtLayer}(1, False)$\;

\BlankLine
\tcp{Paths starting from the input layer}

\ForEach{$NodesList \in PathsAsNodes$}{
    $q_p \gets$ last node of $NodesList$\;
    $i_p \gets StoredPaths[q_p]$\;

    \For{$j \gets 0$ \KwTo $|NodesList|-2$}{
        $k \gets$ index of the parameter connecting
        $NodesList[j]$ and $NodesList[j+1]$\;

        $\PathL[q_p,i_p]
            \gets
            \PathL[q_p,i_p]\theta[k]$\;
    }

    $StoredPaths[q_p] \gets StoredPaths[q_p]+1$\;
}

\BlankLine
\tcp{Paths starting from a bias node at layer $l\geq 2$}

\For{$l \gets 2$ \KwTo $\deep$}{

    $PathsAsNodes
    \gets
    \textsc{ListOfPathsAsNodesAtLayer}(l, True)$\;

    \ForEach{$NodesList \in PathsAsNodes$}{

        $q_p \gets$ last node of $NodesList$\;
        $i_p \gets StoredPaths[q_p]$\;

        $b \gets$ index of parameter connecting the bias node to the first node in $NodesList$\;
        $\PathL[q_p,i_p] \gets \theta_b$\;

        \For{$j \gets 0$ \KwTo $|NodesList|-2$}{

            $k \gets$ index of the parameter connecting
            $NodesList[j]$ and $NodesList[j+1]$\;

            $\PathL[q_p,i_p]
                \gets
                \PathL[q_p,i_p]\cdot\theta[k]$\;
        }

        $StoredPaths[q_p]
            \gets
            StoredPaths[q_p]+1$\;
    }
    
}

\Return{$\PathL$}
\caption{PathLifting}
\label{algo:PathLifting}
\end{algorithm}
\noindent\rule[7pt]{\linewidth}{0.4pt}

\noindent\rule[7pt]{\linewidth}{0.4pt}
\begin{algorithm}[H]
\DontPrintSemicolon
\caption{SkeletonMatrix}
\label{algo:SkeletonMatrix}

\KwData{$LayersList$: ordered list of network layers,
$\theta$: network parameters}
\KwResult{Skeleton tensor $\Skeleton$}

$\Skeleton \gets \mathbf{0}_{\dimy \times \dimPathL_1 \times \dimTheta}$\;
$StoredPaths = \{\}$\;

\ForEach{output node $q$}{
    $StoredPaths[q] \gets 0$\;
}

$PathsAsNodes \gets \textsc{ListOfPathsAsNodesAtLayer}(1, False)$\;

\BlankLine
\tcp{Paths starting from the input layer}

\ForEach{$NodesList \in PathsAsNodes$}{

    $q_p \gets$ last node of $NodesList$\;
    $i_p \gets StoredPaths[q_p]$\;

    \For{$j \gets 0$ \KwTo $|NodesList|-2$}{

        $k \gets$ index of the edge connecting
        $NodesList[j]$ and $NodesList[j+1]$\;

        $\Skeleton[q_p,i_p,k] \gets 1$\;
    }

    $StoredPaths[q_p] \gets StoredPaths[q_p]+1$\;
}

\BlankLine
\tcp{Paths starting from a bias node at layer $l\geq 2$}

\For{$l \gets 2$ \KwTo $\deep$}{

    $PathsAsNodes
        \gets
        \textsc{ListOfPathsAsNodesAtLayer}(l, True)$\;

    \ForEach{$NodesList \in PathsAsNodes$}{

        $q_p \gets$ last node of $NodesList$\;
        $i_p \gets StoredPaths[q_p]$\;
        $b \gets$ index of parameter connecting the bias node to the first node in $NodesList$\;
        
        $\Skeleton[q_p,i_p,b] \gets 1$\;

        \For{$j \gets 0$ \KwTo $|NodesList|-2$}{

            $k \gets$ index of the edge connecting
            $NodesList[j]$ and $NodesList[j+1]$\;

            $\Skeleton[q_p,i_p,k] \gets 1$\;
        }

        $StoredPaths[q_p] \gets StoredPaths[q_p]+1$\;
    }
}

\Return{$\Skeleton$}
\end{algorithm}
\noindent\rule[7pt]{\linewidth}{0.4pt}